\documentclass[sigconf]{aamas}

\usepackage{balance} 

\usepackage{newfloat}
\usepackage{listings}
\DeclareCaptionStyle{ruled}{labelfont=normalfont,labelsep=colon,strut=off} 
\floatstyle{ruled}
\newfloat{listing}{tb}{lst}{}
\floatname{listing}{Listing}

\usepackage{enumitem}
\usepackage{xfrac}
\usepackage{xcolor}
\usepackage{amsmath,amsfonts,amsthm}
\usepackage{cleveref}
\crefformat{footnote}{#2\footnotemark[#1]#3}
\usepackage{bm}
\usepackage{dsfont}
\usepackage{subcaption}
\usepackage{makecell}
\usepackage{colortbl}
\usepackage{tcolorbox}

\definecolor{dkgreen}{rgb}{0,0.6,0}
\definecolor{gray}{rgb}{0.5,0.5,0.5}
\definecolor{mauve}{rgb}{0.58,0,0.82}

\newtheorem{definition}{Definition}
\newtheorem{theorem}{Theorem}

\newtheorem{remark}{Remark}

\makeatletter
\newcommand{\thickhline}{%
    \noalign {\ifnum 0=`}\fi \hrule height 1pt
    \futurelet \reserved@a \@xhline
}
\newcolumntype{"}{@{\hskip\tabcolsep\vrule width 1pt\hskip\tabcolsep}}
\makeatother

\makeatletter
\newcommand{\labelname}[1]{
  \def\@currentlabelname{#1}}%
\makeatother

\setcopyright{ifaamas}
\acmConference[AAMAS '26]{Proc.\@ of the 25th International Conference
on Autonomous Agents and Multiagent Systems (AAMAS 2026)}{May 25 -- 29, 2026}
{Paphos, Cyprus}{C.~Amato, L.~Dennis, V.~Mascardi, J.~Thangarajah (eds.)}
\copyrightyear{2026}
\acmYear{2026}
\acmDOI{}
\acmPrice{}
\acmISBN{}

\acmSubmissionID{95}

\title[LLMs for Designing PB Rules]{Large Language Models for Designing Participatory Budgeting Rules}

\author{Nguyen Thach}
\affiliation{
  \institution{University of Nebraska-Lincoln}
  \city{Lincoln, NE}
  \country{USA}
  }
\email{nate.thach@huskers.unl.edu}

\author{Xingchen Sha}
\affiliation{
  \institution{Northwestern University}
  \city{Evanston, IL}
  \country{USA}}
\email{xingchen.sha@u.northwestern.edu}

\author{Hau Chan}
\affiliation{
  \institution{University of Nebraska-Lincoln}
  \city{Lincoln, NE}
  \country{USA}
  }
\email{hchan3@unl.edu}

\begin{abstract}
Participatory budgeting (PB) is a democratic paradigm for deciding the funding of public projects given the residents' preferences, which has been adopted in numerous cities across the world.
The main focus of PB is designing \emph{rules}, functions that return feasible budget allocations for a set of projects subject to some budget constraint.
Designing PB rules that optimize both utility and fairness objectives based on agent preferences had been challenging due to the extensive domain knowledge required and the proven trade-off between the two notions. 
Recently, large language models (LLMs) have been increasingly employed for automated algorithmic design. 
Given the resemblance of PB rules to algorithms for classical knapsack problems, in this paper, we introduce a novel framework, named LLMRule, that addresses the limitations of existing works by incorporating LLMs into an evolutionary search procedure for automating the design of PB rules.
Our experimental results, evaluated on more than 600 real-world PB instances obtained from the U.S., Canada, Poland, and the Netherlands with different representations of agent preferences, demonstrate that the LLM-generated rules generally outperform existing handcrafted rules in terms of overall utility while still maintaining a similar degree of fairness.
\end{abstract}

\keywords{Participatory Budgeting, Fairness, Large Language Models}

\newcommand{\BibTeX}{\rm B\kern-.05em{\sc i\kern-.025em b}\kern-.08em\TeX}

\begin{document}


\pagestyle{fancy}
\fancyhead{}


\maketitle 


\section{Introduction}

Participatory budgeting (PB) has been widely adopted in many cities across the world, e.g., from the U.S. and the U.K. to regions as diverse as Brazil, Spain, and South Africa \cite{rocke2014framing,gutierrez2018madrid,shah2007public}. 
Originally introduced for deciding the distribution of public funding (or budget) to public projects for residents given their preferences in a democratic manner \cite{cabannes2004participatory}, PB can also be applied to electing a representative committee (e.g., parliament) for voters \cite{faliszewski2017multiwinner}, selecting validators in consensus protocols (e.g., blockchain) \cite{cevallos2021verifiably}, and locating public facilities \cite{skowron2016finding}. 
In the voting stage of a typical PB process\footnote{\label{fn:voting}The computational social choice literature has been focusing on the voting stage of PB extensively \cite{rey2023computational}, which is typically preceded by a shortlisting stage \cite{rey2021shortlisting}.}, the social planner seeks to allocate a given budget to a set of projects (e.g., building new parks and clinics, installing public toilets, and improving street conditions) for serving agents from a region by considering their preferences on the projects and optimizing the overall utility based on agent preferences. 
In addition to utility considerations, the notion of \emph{fairness} (e.g., \emph{proportionality}), in which each agent should have roughly equal influence on the allocation outcome, has been emphasized to align with the spirit of democratic decision making \cite{peters2021proportional,los2022proportional,brill2023proportionality}. 





\paragraph{Existing Efforts.}



Through immense domain expertise, almost all existing studies on PB aimed to manually design the following two types of PB rules and their corresponding implementation algorithms: 
(i) \emph{utilitarian PB rules}, which return optimal or nearly optimal budget allocations subject to some utility objective \cite{talmon2019framework,baumeister2020irresolute,sreedurga2022maxmin}, and 
(ii) \emph{proportional PB rules}, which find budget allocations provably satisfying certain fairness properties \cite{aziz2018proportionally,peters2021proportional,los2022proportional}.
Unfortunately, it has been shown that a trade-off exists between notions of utility and fairness \cite{lackner2020utilitarian,fairstein2022welfare}---PB rules that focus on utility maximization (resp. fairness) are far from being fair (resp. optimal) with respect to the budget allocations.
In fact, proportional rules may return budget allocations with severely compromised utility objective values even for relaxed definitions of fairness \cite{rey2023computational}. 
For example, the Method of Equal Shares (MES) \cite{peters2021proportional}, one of the most recognized proportional rules, finds budget allocations that satisfy strong fairness properties (e.g., EJR in special settings \cite{peters2021proportional}) but typically does not spend the entire budget, which naturally results in a significant drop in the overall utility. 
Thus, in this paper, we strive to address the following question.




\begin{quote}
\textit{Can we design PB rules focusing on budget allocations that are empirically optimal for a given utility objective while being approximately fair?}
\end{quote}


\paragraph{Our Approach and Associated Challenges.}
Large language models (LLMs) have shown remarkable capabilities in algorithmic design \cite{liu2024systematic} and programming \cite{jiang2024survey}.
Notably, integrating LLMs into an evolutionary search procedure to iteratively refine heuristic algorithms has yielded impressive results in various combinatorial optimization problems such as traveling salesman problems and knapsack problems \cite{romera2024mathematical,liu2024evolution,zheng2025monte}.
The idea from these works is to maintain a set of heuristics (coded in some programming language) with good optimization performance on an evaluation dataset of problem instances and iteratively prompt LLMs to generate new heuristics using existing ones as references.
Because the problem of maximizing the utility objective subject to the budget constraint resembles classical knapsack problems \cite{aziz2020participatory}, leveraging LLMs to design PB rules via the stated idea appears to be a plausible approach to our question.

There are, however, three challenges to address when employing existing LLM-based methods. 
(Challenge I) These methods only consider single-objective problems\footnote{By objective, we mean an explicit function inherent to the problem of interest, hence works (e.g., \cite{yao2024multi}) that consider empirical runtime as additional objectives do not apply.}, whereas we are optimizing for both utility and fairness objectives.
(Challenge II) Previous LLM-based methods require a fine-grained performance measure that can capture small improvements during search for ensuring convergence at desirable algorithms \cite{romera2024mathematical}. 
In our case, while most utility objectives can provide a rich enough performance measure, since equivalent objectives for quantifying how fair the allocations returned from PB rules are (with respect to some fairness definition) have yet to exist, 
it is unclear how fairness properties can be effectively enforced for LLM-generated PB rules.
(Challenge III) Existing LLM-based methods assume the availability of an efficient evaluation procedure such that the performance of the generated algorithms can be quickly computed.
Unfortunately, it is well-documented that verifying whether a given budget allocation satisfies standard fairness properties is intractable (i.e., coNP-complete \cite{aziz2017justified,aziz2018complexity}).
Therefore, even when a suitable fairness objective addressing Challenge II exists, using it to evaluate the ``fairness'' of LLM-generated PB rules by computing its value for all training problem instances would naturally be intractable as well, rendering the evaluation procedure prohibitively expensive. 
\paragraph{Contributions.}
We aim to address the stated three challenges and thereby enable the use of LLMs for designing PB rules. 
More specifically, our contributions are as follows:
\begin{itemize}
    \item In response to Challenge II, we introduce a novel fairness objective inspired by Strong-EJR, a desirable fairness property in PB literature \cite{chandak2024proportional}, to which we refer as \emph{Strong-EJR approximation}, to quantify the degree of fairness of the budget allocations returned from a given LLM-generated PB rule. To account for Challenge III, we define this objective based on our provably equivalent definition of Strong-EJR, the verification of which can be done with reduced complexity. 
    
    \item To further address Challenge III, we devise techniques to speed up the computation of Strong-EJR approximation, resulting in a linear-time complexity with respect to the number of agents that is typically large in practice. 
    
    \item Equipped with the aforementioned tools, we develop the first LLM-based framework, named LLMRule, for automating the design of heuristic PB rules that are essentially fair (i.e., yield high values of Strong-EJR approximation) without significant compromise on the overall utility. We tackle Challenge I in the process by introducing a penalty term based on Strong-EJR approximation in the fitness function, originally defined with respect to \emph{utilitarian social welfare} \cite{fluschnik2019fair} as the utility objective, from which the performance of LLM-designed PB rules is measured. Additionally, LLMRule requires minimal domain knowledge and hyperparameter tuning (e.g., prompt design), which alleviates the manual efforts involved in contemporary PB studies. 
    
    \item Using more than 600 real-world PB instances across Europe and North America, we empirically demonstrate that LLMRule outperforms existing proportional rules in terms of utilitarian social welfare while maintaining a similar degree of fairness with respect to Strong-EJR approximation on classical settings \cite{los2022proportional,brill2023proportionality} where agent preferences are represented as \emph{approval ballots} (i.e., each agent submit a subset of projects they approve of) and the utility function is defined as (i) the number and (ii) the cost of selected and approved projects. Furthermore, the good performance of LLMRule also translates to more general settings with \emph{cardinal ballots} (i.e., each agent submits a score for each project). 
    
\end{itemize}

\paragraph{Outline.}
In the remainder of this paper, we provide the preliminaries for our considered PB problems in Section \ref{sec:prelim}. Section \ref{sec:method} details our proposed LLMRule framework. We validate its efficacy in Section \ref{sec:exp}. 
We discuss related work in Section \ref{sec:related} and conclude our work in Section \ref{sec:conclu}. 
Readers can find complete related work in Appendix \ref{app:related}, additional theoretical results in Appendix \ref{app:pjr}, complete implementation details in Appendices \ref{app:imp}--\ref{app:designed-rules} (which include all missing figures and listings), additional experiments involving synthetic data in Appendix \ref{app:exp}, and limitations and future work in Appendix \ref{app:limit}.








\section{Preliminaries}\label{sec:prelim}

We present the classical settings of PB that are most studied in the literature \cite{rey2023computational}.
Let $I = \langle \mathcal{P}, c, b\rangle$ be a PB instance, where $\mathcal{P}=\{p_1,\ldots,p_m\}$ is the set of $m$ projects, $c:\mathcal{P}\rightarrow\mathbb{R}_{>0}$ is the cost function that assigns a cost $c(p)\in\mathbb{R}_{>0}$ to every project $p\in\mathcal{P}$, and $b\in\mathbb{R}_{>0}$ is the budget.
Note that from a broader perspective, $I$ represents the voting stage of a PB process$^{\ref{fn:voting}}$.
For any subset of projects $P\subseteq\mathcal{P}$, we denote $c(P)=\sum_{p\in P}c(p)$ as its total cost.

Let $\mathcal{N}=\{1,\ldots,n\}$ be the set of $n$ agents or voters involved in $I$.
Each agent is associated with a \emph{ballot} that represents their preferences over the projects in $\mathcal{P}$.
In the most general setting, the agents are asked to submit a score for all projects, known as \emph{cardinal ballots}.
A cardinal ballot for an agent $i\in\mathcal{N}$ is defined as $A_i:\mathcal{P}\rightarrow\mathbb{R}_{\ge0}$. 
Because collecting agent preferences from this type of ballot can be laborious especially when $m$ is large \cite{iyengar2000choice}, most real-life PB processes use \emph{approval ballots} \cite{aziz2020participatory}, where agents are only asked to submit a subset of projects they approve of.
In other words, it is a special case of cardinal ballots: an approval ballot for $i$ is defined as $A_i:\mathcal{P}\rightarrow\{0,1\}$, where for any $p\in\mathcal{P}$, $A_i(p)=1$ indicates that agent $i$ approves of project $p$ and $A_i(p)=0$ otherwise.
The vector $\bm{A}=(A_1,\ldots,A_n)$ forms a \emph{profile} of ballots.


Based on this profile, the social planner seeks to find an allocation outcome of $I$, which is a budget allocation $\pi\subseteq\mathcal{P}$ such that $c(\pi)\le b$.\footnote{In this paper, we focus on indivisible or discrete PB in which projects can only be fully funded or not at all.} 
A PB rule $f$ is a function that takes an instance $I = \langle \mathcal{P}, c, b\rangle$ and a profile $\bm{A}$ as input and outputs a budget allocation $\pi$.
Note that in general, a PB rule returns a subset of all feasible budget allocations i.e., $f(I,\bm{A})\subseteq\{\pi\subseteq\mathcal{P}\mid c(\pi)\le b\}$.
We focus on \emph{resolute} rules that always return a single budget allocation (via some tie-breaking strategy) and hence, with a slight abuse of notation, denote the output $\{\pi\}$ of a resolute rule by just $\pi$.
Following existing utility maximization studies \cite{talmon2019framework,baumeister2020irresolute}, our goal is to find a resolute PB rule $f$ that maximizes the overall utility of the agents, or formally the \emph{utilitarian social welfare} \cite{fluschnik2019fair}. For cardinal ballots, this objective is defined as 
\begin{equation}\label{eq:util-cardinal}
    \omega(I,\bm{A},f)=\sum_{i\in\mathcal{N}}\sum_{p\in f(I,\bm{A})}A_i(p)
\end{equation}
where we follow the standard \emph{additivity} assumption \cite{peters2021proportional} that equates the score of a budget allocation for an agent to the sum of the scores of the projects it contains.
For approval ballots, the objective is
\begin{equation}\label{eq:util-app}
    \omega[sat](I,\bm{A},f)=\sum_{i\in\mathcal{N}}sat_i(f(I,\bm{A}))
\end{equation}
where $sat_i(\pi)=sat(\{p\in\pi\mid A_i(p)=1\})$ and $sat:2^{\mathcal{P}}\rightarrow\mathbb{R}_{\ge0}$ is a satisfaction function for translating a budget allocation into a satisfaction level for the agents given their approval ballots \cite{brill2023proportionality}.
We focus on two standard (additive) satisfaction functions \cite{talmon2019framework}:
\begin{itemize}
    \item Cardinality satisfaction function: $sat^{card}(P)=|P|$
    \item Cost satisfaction function: $sat^{cost}(P)=c(P)$.
\end{itemize}

At the same time, we also want $f$ to be \emph{fair}, i.e., satisfies some fairness property $\mathcal{X}$, in the sense that for every instance $I$ and profile $\bm{A}$, the allocation outcome of the (resolute) rule $f(I, \bm{A})$ is guaranteed to satisfy $\mathcal{X}$.
Formalizing fairness properties requires the following definitions of cohesive groups \cite{peters2020proportionality}. 
We start with cardinal ballots.
\begin{definition}[($\alpha, P$)-Cohesive Groups]\label{def:cohegroup-cardinal}
    Given an instance $I = \langle \mathcal{P}, c, b\rangle$ and a profile $\bm{A}$ of cardinal ballots, a non-empty group of agents $N\subseteq \mathcal{N}$ is said to be ($\alpha, P$)-cohesive, for a function $\alpha: \mathcal{P}\rightarrow \mathbb{R}_{\ge0}$ and a set of projects $P\subseteq \mathcal{P}$, if the following two conditions are satisfied:
    \begin{itemize}
        \item $\alpha(p) \le A_i(p)$ for all $i \in N$ and $p \in P$ , that is, $\alpha$ is lower-bounding the score of the agents in $N$;
        \item $\frac{|N|}{n}\cdot b \ge c(P)$, that is, $N$’s share of the budget is enough to afford $P$.
    \end{itemize}
\end{definition}
For simplicity, we restrict $\alpha$ to be the minimum positive score submitted by any agent in $N$ for project $p\in \mathcal{P}$.\footnote{\label{fn:nontrivial-pb}We only consider nontrivial PB instances with at least one cohesive group, $\max_{i\in \mathcal{N}}A_i(p)>0$ for any $p\in\mathcal{P}$, and $\sum_{p\in\mathcal{P}}c(p)>b$.} 
Similarly, cohesive groups for approval ballots are defined as follows.
\begin{definition}[$P$-Cohesive Groups]\label{def:cohegroup-app}
    Given an instance $I = \langle \mathcal{P}, c, b\rangle$ and a profile $\bm{A}$ of approval ballots, a non-empty group of agents $N\subseteq \mathcal{N}$ is said to be $P$-cohesive, for a set of projects $P\subseteq \mathcal{P}$, if the following two conditions are satisfied:
    \begin{itemize}
        \item $A_i(p)=1$ for all $i \in N$ and $p \in P$ , that is, every agent in $N$ approves all projects in $P$;
        \item $\frac{|N|}{n}\cdot b \ge c(P)$, that is, $N$’s share of the budget is enough to afford $P$.
    \end{itemize}
\end{definition}

Ideally, all agents in $N$ should deserve at least as much satisfaction as they all agree $P$ would offer them (captured by $\bm{A}$).
The following fairness property, termed \emph{strong extended justified representation} (Strong-EJR) \cite{aziz2017justified,chandak2024proportional}, captures this idea. 

\begin{definition}[Strong-EJR for Cardinal Ballots]\label{def:strong-ejr-cardinal}
    Given an instance $I = \langle \mathcal{P}, c, b\rangle$ and a profile $\bm{A}$ of cardinal ballots, a budget allocation $\pi$ is said to satisfy strong extended justified representation (Strong-EJR) if for all $P \subseteq \mathcal{P}$ and all ($\alpha, P$)-cohesive groups $N$, and all $i^{\star} \in N$, we have:
    \begin{equation*}
        \sum_{p\in\pi}A_{i^{\star}}(p)\ge \sum_{p\in P}\min_{i\in N}A_i(p).
    \end{equation*}
\end{definition}


\begin{definition}[Strong-EJR for Approval Ballots]\label{def:strong-ejr-app}
    Given an instance $I = \langle \mathcal{P}, c, b\rangle$, a profile $\bm{A}$ of approval ballots, and a satisfaction function $sat$, a budget allocation $\pi$ is said to satisfy Strong-EJR[$sat$] if for all $P \subseteq \mathcal{P}$ and all $P$-cohesive groups $N$, we have $sat_i(\pi)\ge sat_i(P)$ for all agents $i \in N$.
\end{definition}
While appealing, budget allocations satisfying Strong-EJR may not exist for a given $I$ in general \cite{aziz2017justified,chandak2024proportional}.
Hence, several weaker fairness properties exist in the literature, notably extended justified representation (EJR) \cite{aziz2017justified,peters2021proportional}, which requires only one member of each cohesive group to enjoy the deserved satisfaction. 



\begin{definition}[EJR for Cardinal Ballots]\label{def:ejr-cardinal}
    Given an instance $I = \langle \mathcal{P}, c, b\rangle$ and a profile $\bm{A}$ of cardinal ballots, a budget allocation $\pi$ is said to satisfy extended justified representation (EJR) if for all $P \subseteq \mathcal{P}$ and all ($\alpha, P$)-cohesive groups $N$, there exists $i^{\star} \in N$ such that:
    \begin{equation*}
        \sum_{p\in\pi}A_{i^{\star}}(p)\ge \sum_{p\in P}\min_{i\in N}A_i(p).
    \end{equation*}
\end{definition}
\begin{definition}[EJR for Approval Ballots]\label{def:ejr-app}
    Given an instance $I = \langle \mathcal{P}, c, b\rangle$, a profile $\bm{A}$ of approval ballots, and a satisfaction function $sat$, a budget allocation $\pi$ is said to satisfy EJR[$sat$] if for all $P \subseteq \mathcal{P}$ and all $P$-cohesive groups $N$, there exists $i \in N$ such that $sat_i(\pi)\ge sat_i(P)$.
\end{definition}

\section{The LLMRule Framework}\label{sec:method}

We start this section with an overview of our proposed framework, which is inspired by the popular EoH method \cite{liu2024evolution} in automated algorithmic design. 
LLMRule maintains a population of $h$ heuristic rules, denoted as $R=\{f_1,\ldots,f_h\}$.
It adopts an evolutionary search procedure, iteratively searching for rules that yield better trade-offs between utility and fairness objectives. 
Each rule $f_j\in R$ is evaluated on a set of problem instances and assigned a fitness value $q(f_j)$ that indicates its probability of being selected as reference for generating new rules.
The remainder of this section is organized as follows: Section \ref{subsec:sejr-verify} provides our provably equivalent definition of Strong-EJR for reducing the complexity of its verification (Challenge III); using this definition, Section \ref{subsec:fairness-approx} introduces Strong-EJR approximation, a novel fairness objective used for informing the design of empirically fair utilitarian PB rules (Challenge II); Section \ref{subsec:fairness-verify} describes our techniques for efficiently computing Strong-EJR approximation, which in turn expedites the training of LLMRule (Challenge III); Section \ref{subsec:fairness-eval} elaborates on LLMRule, in which the fitness function $q$ for evaluating PB rules in terms of both objectives is defined (Challenge I).

\subsection{Efficient Strong-EJR Verification}\label{subsec:sejr-verify}
Based on Definitions \ref{def:strong-ejr-cardinal} and \ref{def:strong-ejr-app}, na\"{i}vely verifying Strong-EJR for a given $\pi$ requires inspecting all possible combinations of projects $P$ and groups of agents $N$ pertaining to $P$, followed by comparing the satisfaction between $P$ and $\pi$ for each agent in $N$. Clearly, this exhaustive process takes $O(2^{m+n}\times n)$, which is intractable.
In the following, we lay the groundwork for reducing this complexity.
For brevity, we use $P$-cohesive groups to also refer to ($\alpha, P$)-cohesive groups for cardinal ballots when the context is clear, where $\alpha$ is fixed to the smallest positive score agents can assign to projects.

\begin{definition}[Maximal $P$-cohesive group]
    A $P$-cohesive group is maximal, denoted as $N^{\max}$, if $N^{\max}\supseteq N$ for any other $P$-cohesive group $N$.
\end{definition}

\begin{theorem}\label{thm:maximal-group}
    Given $P\subseteq \mathcal{P}$ and two $P$-cohesive groups $N$ and $N'$ where $N'\subseteq N$, if $sat_i(\pi)\ge sat_i(P)$ for all $i\in N$, then $sat_{i'}(\pi)\ge sat_{i'}(P)$ for all $i'\in N'$.\\Likewise, if $\sum_{p\in\pi}A_{i^{\star}}(p)\ge \sum_{p\in P}\min_{i\in N}A_i(p)$ for all $i^{\star}\in N$, then $\sum_{p\in\pi}A_{i'}(p)\ge \sum_{p\in P}\min_{i\in N'}A_i(p)$ for all $i'\in N'$.
\end{theorem}
\begin{proof}
    By contradiction, assume there exists $i'\in N'$ such that $sat_{i'}(\pi)<sat_{i'}(P)$. Since $N'\subseteq N$, it follows that there exists $i\in N$ such that $sat_{i}(\pi)<sat_{i}(P)$, which is a contradiction. The same applies to the statement for cardinal ballots.
\end{proof}
Theorem \ref{thm:maximal-group} implies that when verifying Strong-EJR, only checking the maximal $P$-cohesive group $N^{\max}$ for a given $P\subseteq\mathcal{P}$ is sufficient, which leads us to the following revision of Definitions \ref{def:strong-ejr-cardinal} and \ref{def:strong-ejr-app}.

\begin{definition}[Strong-EJR for Cardinal Ballots (Updated)]\label{def:strong-ejr-cardinal-new}
    Given an instance $I = \langle \mathcal{P}, c, b\rangle$ and a profile $\bm{A}$ of cardinal ballots, a budget allocation $\pi$ is said to satisfy strong extended justified representation (Strong-EJR) if for all $P \subseteq \mathcal{P}$ and all $i^{\star} \in N^{\max}$, the corresponding \textbf{maximal} ($\alpha, P$)-cohesive group, we have:
    \begin{equation*}
        \sum_{p\in\pi}A_{i^{\star}}(p)\ge \sum_{p\in P}\min_{i\in N}A_i(p).
    \end{equation*}
\end{definition}


\begin{definition}[Strong-EJR for Approval Ballots (Updated)]\label{def:strong-ejr-app-new}
    Given an instance $I = \langle \mathcal{P}, c, b\rangle$, a profile $\bm{A}$ of approval ballots, and a satisfaction function $sat$, a budget allocation $\pi$ is said to satisfy Strong-EJR[$sat$] if for all $P \subseteq \mathcal{P}$, we have $sat_i(\pi)\ge sat_i(P)$ for all agents $i \in N^{\max}$, the corresponding \textbf{maximal} $P$-cohesive group.
\end{definition}

Finding $N^{\max}$ for approval ballots and for cardinal ballots with $\alpha=1$ (i.e., by searching in $\bm{A}$ the maximal set of agents that approve or assign positive scores to all projects in $P$) only takes $O(n|P|)$ or $O(nm)$, 
meaning the verification of Strong-EJR can be effectively done in $O(2^m\times(mn+n))$ or $O(2^m\times mn)$.
Notably, the exponentiation term with $n$ is dropped, which is substantial as $n\gg m$ in realistic PB processes \cite{faliszewski2023pabulib,boehmer2024guide}.

\begin{remark}[Applicability to EJR Verification]\label{rem:ejr}
    Theorem \ref{thm:maximal-group} does not hold for existential quantifiers. Formally, if there exists $i\in N$ such that $sat_i(\pi)\ge sat_i(P)$, then we are not guaranteed to have $i'\in N'$ such that $sat_{i'}(\pi)\ge sat_{i'}(P)$, particularly when $i\in N\setminus N', N'\subset N$. The same applies to cardinal ballots. Hence, EJR verification does not enjoy similar complexity reduction.
\end{remark}

Remark \ref{rem:ejr} is consistent with \cite{aziz2017justified}, which states that the verification of EJR given a budget allocation is coNP-complete.
Additionally, our remark also holds for relaxations of EJR, e.g., EJR up to one project (EJR-1) \cite{peters2021proportional} and EJR up to any project (EJR-X) \cite{brill2023proportionality}.
Please refer to Appendix \ref{app:pjr} for our similar remark on proportional justified representation (PJR) \cite{sanchez2017proportional}, an even weaker variant of Strong-EJR.

\subsection{Strong-EJR Approximation}\label{subsec:fairness-approx}

Given that Strong-EJR can now be verified much more efficiently, we introduce the following fairness objective based on Definitions \ref{def:strong-ejr-cardinal-new} and \ref{def:strong-ejr-app-new}.
Given a PB instance $I$ and a profile $\bm{A}$, we define \emph{Strong-EJR approximation} achieved by a resolute PB rule $f: (I, \bm{A})\rightarrow \pi$ as
\begin{equation}\label{eq:strong-ejr-approx-card}
     \phi(I,\bm{A},f)=\frac{1}{s}\sum_{P\in S}\min_{i^{\star}\in N^{\max}}\min\left\{1,\frac{\sum_{p\in f(\cdot)}A_{i^{\star}}(p)}{\sum_{p\in P}\min_{i\in N}A_i(p)}\right\}
\end{equation} 
for cardinal ballots and
\begin{equation}\label{eq:strong-ejr-approx}
     \phi[sat](I,\bm{A},f)=\frac{1}{s}\sum_{P\in S}\min_{i\in N^{\max}}\min\left\{1,\frac{sat_i(f(\cdot))}{sat_i(P)}\right\}
\end{equation}
for approval ballots where $S$ is the set of all maximal $(\alpha,P)$-/$P$-cohesive groups $N^{\max}$ from $(I,\bm{A})$ and $s=|S|$, which is bounded by $[1,2^{m})$.$^{\ref{fn:nontrivial-pb}}$
Clearly, $\phi\in[0,1]$, with 1 being perceived as ``fair''.\footnote{The upper bound is not necessarily tight since Strong-EJR is unsatisfiable in general \cite{aziz2017justified}.} 

Unfortunately, by Remark \ref{rem:ejr} (and Remark \ref{rem:pjr} in Appendix \ref{app:pjr}), similar definitions for weaker fairness properties such as EJR (and PJR) are not applicable. 









\subsection{Accelerating LLMRule with Efficient Computation of Strong-EJR Approximation}\label{subsec:fairness-verify}

Following Section \ref{subsec:sejr-verify}, computing Strong-EJR approximation also takes $O(2^m\times mn)$.
Notice this procedure can be split into two steps:
\begin{equation}\label{eq:veri-split}
    O(2^m\times mn)=O(\underbrace{2^m\times mn}_{\text{(i)}} + \underbrace{2^m\times n}_{\text{(ii)}})
\end{equation}
where (i) is required for finding all $s$ maximal cohesive groups and (ii) for computing the satisfaction ratio (term in braces from Equations \ref{eq:strong-ejr-approx-card} and \ref{eq:strong-ejr-approx}) for all agents from them.
Because (i) does not require $f$, it can be done separately beforehand i.e., prior to training LLMRule.
Thus, the complexity can be effectively reduced to $O(2^m\times n)$. 


We now seek to improve the exponential term with $m$ knowing that not every $P$ has at least one cohesive group.
We now refer to $N^{\max}$ as simply $N$ in the remainder of this section for readability.
We make the following observation based on the definition of cohesive groups (Definitions \ref{def:cohegroup-cardinal} and \ref{def:cohegroup-app}). 
\begin{align*}
    \frac{|N|}{n} &\cdot b \ge c(P)\\
    \gamma = \frac{|N|}{n} &\cdot \frac{b}{c(P)} \ge 1.
\end{align*}
In other words, the larger $\gamma$, the more ``deserving'' $N$ gets to have some satisfaction in the final allocation outcome.
From this observation, we can sort $P$-cohesive groups $N$ by descending $\gamma$, where ties are broken based on the cumulative approval score $\prod_{p\in P} app(p, \bm{A})$, where $app(p, \bm{A})=|\{i\in\mathcal{N}\mid A_i(p)=1\}|$. 
Notice $|N|/n$ is the frequency of $P$ in $\bm{A}$, which is consistent with the definition of \emph{support} in frequent itemset mining \cite{han2007frequent}. 
Hence, we can use any algorithm in this field to obtain a sorted list of maximal cohesive groups. 
Let $supp_P$ be the support for $P$ and consider the standard Apriori algorithm \cite{huang2000fast}.
Because $supp_P=\frac{|N|}{n}\ge\frac{c(P)}{b}$ for any $P$-cohesive group $N$, the minimum support is $\min_{p\in\mathcal{P}}c(p)/b$, i.e., the cost of the cheapest project in $\mathcal{P}$ over $b$. 
We use this value as input to the Apriori algorithm. 
Its output is a set of 2-tuples $(P, supp_P)$, which is used to compute $\gamma$, then search for $N$ (which takes $O(mn)$ as discussed). 
Given that Apriori runs in $O(2^m)$ \cite{huang2000fast}, step (i) in Equation \ref{eq:veri-split} still maintains a complexity of $O(2^m+2^m\times mn)=O(2^m\times mn)$ in total.
Once the resulting list of $s$ maximal cohesive groups is sorted, given $f$, we can proceed to step (ii) in Equation \ref{eq:veri-split} i.e., compute $\phi$ starting from $P$-cohesive group $N$ with largest $\gamma$.
Overall, computing Strong-EJR approximation in this way while training LLMRule takes only $O(s\times n)< O(2^m\times n)$. 
When considering instances with large $m$ and approval ballots in particular, an upper limit of top-$\sigma$ maximal cohesive groups with $\sigma<s\ll 2^m$ can be specified in exchange for exhaustiveness.

\subsection{LLMs for Designing PB Rules}\label{subsec:fairness-eval}



\paragraph{Rule Representation.}
We first describe the components to represent a PB rule.
\begin{enumerate}
    \item The rule description comprises a few sentences in natural language. It is created by LLMs and encapsulates a high-level thought. An example is provided in Figure \ref{fig:learned-rule1} (top).
    \item The code block is an implementation of the PB rule (Figure \ref{fig:learned-rule1}, bottom). It should follow a pre-defined format (see Listing \ref{lst:code-template-cardinal}) so that it can be identified and seamlessly integrated into the evolutionary framework. In the experiments, we choose to implement it as a Python function, though any programming language should work.
    \item Each rule is assigned a fitness value to represent its priority in the population, which is used for selection and population management. 
    We elaborate our definition of fitness based on utility and fairness objectives shortly. 
\end{enumerate}

\paragraph{Evolutionary Search.}
To begin the evolution process, we inform the LLM of the PB problem of interest and instruct it to design a new rule by first presenting the description of the generated rule and the corresponding code block (see Figure \ref{fig:init_prompt-design}, top).
We repeat $h$ times to generate $h$ initial rules.
From these initial rules, offspring rules are generated by prompting the LLM with specific instructions, called \emph{prompt strategies}, that either explore or modify existing rules.
Exploration strategies focus more on exploring uncharted area in the space of heuristics by conducting crossover-like operators on parent heuristics (Figure \ref{fig:init_prompt-design}, left).
Modification strategies refine a parent heuristic by modifying, adjusting parameters, and removing redundant parts (Figure \ref{fig:init_prompt-design}, right).
For selecting parent rule(s) from a population $R$ of $h$ rules when generating an offspring rule, we first rank them according to their fitness, then randomly select one or multiple rules $f_j$ with probability $\text{Pr}(f_j)\propto1/(\rho_j+h)$, where $\rho_j$ is its rank in $R$. 
Once $h$ offspring rules are generated, we select the $h$ best rules from the current population (comprising parent and offspring rules) to form a population for the next iteration.
The evolution process terminates when the evaluation resource (e.g., number of populations) is depleted.

\paragraph{Prompt Refinement.}
Since prompt strategies dictate the diversity in the population and play a vital role in producing high-quality heuristics \cite{liu2024systematic}, thoughtful prompt engineering is necessary, which often requires certain degree of domain knowledge (especially for modification ones).
To alleviate this, we incorporate a prompt refinement module wherein prompt strategies are iteratively refined as well.
In particular, when the fitness values of the top-$l$ rules in the population remain unchanged for $t$ consecutive units of evaluation resource, which indicates stagnation during the search, we refine the prompt strategies from each type as follows.
\begin{itemize}
    \item \textbf{Exploration:} Given the current set of exploration strategies, ask the LLM to design a new one that is different from them as much as possible (Figure \ref{fig:evol_prompt-design}, top).
    \item \textbf{Modification:} Given the current set of modification strategies, ask the LLM to design a new one that is different from them as much as possible. (Figure \ref{fig:evol_prompt-design}, bottom).
\end{itemize}
By this means, two new prompt strategies, one from each type, are added after calling this module.
To maintain the size of the set of prompt strategies, during the evolution process, we assign each prompt strategy a score defined as the average fitness value of the top-$d$ individuals produced by it, and keep those with highest scores for generating offspring rules.

\begin{figure}[t]
    \footnotesize
    \centering
    \begin{tcolorbox}[width=0.47\textwidth, colback=blue!5]
        \textbf{Rule Description (generated using the initial E1 prompt strategy)}

        \hphantom\\

        A new strategy will utilize a geometric mean approach combining approval rates and a cost-to-budget proportion, enabling a balanced evaluation for project prioritization.
    \end{tcolorbox}
    \medskip
\begin{lstlisting}[language=Python]
import numpy as np
def priority(project_costs, budget, approval_mat):
    '''
    Args:
    project_costs (np.ndarray): Array of length M storing the project costs.
    budget (int or float): A positive value.
    approval_mat (np.ndarray): An N-by-M matrix where each row contains a sample of binary approvals for the respective projects.

    Returns:
    scores (np.ndarray): Array of priority scores for the projects.
    '''
    approval_rates = np.mean(approval_mat, axis=0)
    cost_to_budget_ratio = project_costs / budget
    scores = np.sqrt(approval_rates) * (1 / (1 + cost_to_budget_ratio))

    return scores
\end{lstlisting}
    \caption{High-level description of an LLM-generated PB rule for approval ballots with $sat=sat^{cost}$ and its associated code.}
    \label{fig:learned-rule1}
\end{figure}

\paragraph{Fitness Evaluation.}

The evaluation of each newly generated PB rule $f$ involves running it on a set of $K$ problem instances from which its fitness value is computed.
Since we aim to maximize both utility (Equations \ref{eq:util-cardinal} and \ref{eq:util-app}) and fairness (Equations \ref{eq:strong-ejr-approx-card} and \ref{eq:strong-ejr-approx}) objectives, the associated fitness function should include both terms.
Unfortunately, existing works on LLMs have yet to consider multi-objective combinatorial optimization to our best knowledge.
In response, we circumvent this challenge by introducing a penalty term in the fitness function whenever $f$ is not empirically fair.
Formally, given a heuristic PB rule $f$ that takes an instance $I$ and a profile $\bm{A}$ of ballots as input and outputs a budget allocation $\pi\subseteq\mathcal{P}$, its fitness value $q(f)$ is defined as\footnote{For brevity, we omit the $[sat]$ parameterization in $\omega$ and $\theta$ for approval ballots.} 
\begin{align}
\begin{split}
    q(f) = \frac{1}{K}\sum^{K}_{k=1} \biggl[\;&\omega'\Bigl(I^{(k)},\bm{A}^{(k)},f(\cdot)\Bigr) \\
    &- \theta\Bigl(\phi\bigl(I^{(k)},\bm{A}^{(k)},f(\cdot)\bigr)\Bigr)\biggr],
\end{split}
\end{align}
where
\begin{equation*}
    \omega'(I,\bm{A},f)=\frac{\omega(I,\bm{A},f)}{\omega(I,\bm{A},\pi^{\text{OPT}})}
\end{equation*}
is the utilitarian social welfare of $f$ relative to the one yielded from the optimal budget allocation $\pi^{\text{OPT}}$ (obtained by computing the welfare maximizer) and
$\theta$ is a function of Strong-EJR approximation $\phi$ yielded from $f$:
\begin{equation}\label{eq:fitness-penalty}
\theta\Bigl(\phi(\cdot)\Bigr) = 
\begin{cases}
    1 & \text{if }\phi<\varepsilon \\
    0 & \text{otherwise},
\end{cases}
\end{equation}
for some value $0\le\varepsilon\le1$.
That is, we penalize $f$'s fitness value if its associated Strong-EJR approximation is less than some predefined threshold.
Because $\omega'$ is bounded within $[0,1]$, a PB rule with a negative fitness value implies that it is not empirically fair (subject to Strong-EJR). 
Setting $\varepsilon<1$ relaxes the fairness constraint.



\section{Experiments}\label{sec:exp}



\subsection{Experimental Setups}

All experiments were conducted under Ubuntu 20.04 on a Linux virtual machine equipped with NVIDIA GeForce RTX 3050 Ti GPU and 12th Gen Intel(R) Core(TM) i7-12700H CPU @2.3GHz.
We chose GPT-4o mini as our pretrained LLM.
The code for our implementation in Python 3.10 is available on GitHub\footnote{\url{https://github.com/AnonyMouse3005/LLM-PB}}.

\paragraph{Dataset.}

We obtain 617 real-world PB instances from Pabulib\footnote{\label{fn:pabulib}\url{https://pabulib.org/}} \cite{faliszewski2023pabulib}, which were held in the U.S., Canada, Poland, and the Netherlands with either approval or cardinal ballots.
We include all instances satisfying the following criteria: having 3 to 25 projects, not ``fully funded''$^{\ref{fn:nontrivial-pb}}$, and not ``experimental'' i.e., synthetic.
Note that Pabulib only includes a special type of cardinal ballots, called cumulative ballots, in which voters are allocated a specific number of points that they can distribute among the projects ($\sum_{p\in\mathcal{P}}A_i(p)\le1$).
For each of the two types of ballots, we further categorize the dataset into small and large PB instances based on $n$, with small instances having fewer than 1,000 voters. 
Table \ref{tab:data} specifies our employed dataset.
Notably, we train LLMRule using the small instances, then test the learned rules on the large instances, i.e., out-of-distribution (OOD) testing.
Due to a plethora of small PB instances with approval ballots, we only use those held in the U.S. for training and the rest for in-distribution (ID) testing. 


\begin{table}[h]
    \centering
    \caption{Number of PB instances obtained from Pabulib. Data used for training LLMRule are shaded in green, and data for OOD testing are unshaded. Approval ballots also include an ID test set (shaded in gray).} 
    \begin{tabular}{l|cc|c}
        Ballot type & \multicolumn{2}{c|}{Small ($n<$1,000)} & Large ($n\ge$1,000) \\\hline
        Cardinal & \multicolumn{2}{c|}{\cellcolor{green!25}63} & 122 \\
        Approval & \cellcolor{green!25}74 & \cellcolor{gray!25}133 & 225\\
    \end{tabular}
    \label{tab:data}
\end{table}





\paragraph{Baselines.}
We consider the following well-studied PB rules in the literature that were implemented in the Pabutools library\footnote{\url{https://github.com/COMSOC-Community/pabutools}} \cite{faliszewski2023pabulib}. 
\begin{itemize}
    \item \textsc{MaxUtil}: the Additive Utilitarian Welfare Maximizer \cite{talmon2019framework}, which returns $\pi^{\text{OPT}}$. 
    
    \item \textsc{GreedUtil}: the Greedy Utilitarian Rule \cite{talmon2019framework}. 
    \item MES: the Method of Equal Shares \cite{peters2021proportional}. For approval ballots, we run MES for $sat=sat^{cost}$ and $sat=sat^{card}$, denoted respectively as MES[$sat^{cost}$] and MES[$sat^{card}$]. Because MES typically does not spend all the given budget, i.e., it is not \emph{exhaustive}, we additionally run MES with three completion methods: `Add1' (special case of `Exh2' in \cite{peters2021proportional}) along with `Add1U' \cite{faliszewski2023pabulib} and `Add1UM' that respectively use \textsc{GreedUtil} and \textsc{MaxUtil} for exhausting the remaining budget. 
\end{itemize}
For approval ballots, we also consider two standard proportional rules other than MES: 
\begin{itemize}
    \item \textsc{SeqPhrag} \cite{los2022proportional}: the Sequential Phragm\'{e}n Rule. 
    \item \textsc{MaximinSupp} \cite{aziz2018proportionally}: the Maximin Support Rule, also known as Generalised Sequential Phragm\'{e}n Rule. 
\end{itemize}
By default, these rules are irresolute, hence we make them resolute by implementing lexicographic tie-breaking. 



In terms of which rules satisfy which fairness property, for cardinal ballots, MES satisfies EJR-1 \cite{peters2021proportional} and PJR-1 \cite{los2022proportional}. For approval ballots, 
\begin{itemize}
    \item Strong-EJR[$sat$]: None \cite{rey2023computational}
    \item EJR[$sat$]: MES[$sat^{card}$] \cite{peters2021proportional}
    \item EJR-X[$sat$]: MES[$sat$] for any DNS function $sat$\footnote{\label{fn:dns-func}Both $sat^{cost}$ and $sat^{card}$ are DNS functions (defined and proved in \cite{brill2023proportionality}).} \cite{brill2023proportionality}
    \item EJR-1[$sat$]: MES[$sat$] for any additive $sat$ \cite{peters2021proportional}
    \item PJR[$sat$]: None
    \item PJR-X[$sat$]: MES[$sat$] for any DNS function $sat$$^{\ref{fn:dns-func}}$, \textsc{SeqPhrag}, \textsc{MaximinSupp} \cite{brill2023proportionality}
\end{itemize}
Note that we do not consider the theoretically appealing ``greedy cohesive rule'' introduced in \cite{peters2021proportional}, which has been shown to satisfy all listed fairness properties except Strong-EJR, since it runs in time exponential in $n$ and therefore highly impractical.


\paragraph{Implementation Details.}
We consider three problem settings: approval ballots with (i) $sat=sat^{cost}$ and (ii) $sat=sat^{card}$, and (iii) cardinal ballots.
That is, the utility objectives are defined as (i) $\omega[sat^{cost}]$ and (ii) $\omega[sat^{card}]$ (Equation \ref{eq:util-app}), and (iii) $\omega$ (Equation \ref{eq:util-cardinal}).
In each setting, we use LLMRule to design greedy rules, in which the task for the LLM is to design a priority function for assigning a score to each project. This function takes a list of project costs, the budget constraint, and the ballot profile as input and outputs the scores for the projects.
Using these scores, the allocation outcome $\pi\subseteq\mathcal{P}$ is constructed by greedily adding the projects until none of the remaining projects are affordable given the current budget. 
Please refer to Appendix \ref{app:imp} for further implementation details.

Our motivations for designing greedy rules are twofold. \emph{Flexibility}: Greedy rules can also serve as effective completion methods (e.g., `Add1') \cite{faliszewski2023pabulib} for proportional rules that are non-exhaustive such as MES. \emph{Practicality}: \textsc{GreedUtil} is in fact the most widely used rule in the real world \cite{rey2023computational}. 

\begin{table}[b]
    \centering
    \small
    \caption{Results for approval ballots with (top) $sat=sat^{cost}$ and (bottom) $sat=sat^{card}$ on ID and OOD test sets. Utilitarian/proportional rules are respectively shaded in gray/green.} 
    \begin{tabular}{l|c|c|c|c}
        \thickhline
        PB rule & \makecell{$\omega'[sat^{cost}]$\\ (ID)} & \makecell{$\phi$\\ (ID)} & \makecell{$\omega'[sat^{cost}]$\\ (OOD)} & \makecell{$\phi$\\ (OOD)} \\\hline
        \cellcolor{gray!25}\textsc{MaxUtil} & 1 & 0.643 & 1 & 0.536 \\
        \cellcolor{gray!25}\textsc{GreedUtil} & 0.978 & 0.750 & 0.971 & 0.686 \\\hline
        \cellcolor{green!25}\textsc{SeqPhrag} & 0.710 & 0.955 & 0.695 & 0.956 \\
        \cellcolor{green!25}\textsc{MaximinSupp} & 0.696 & 0.956 & 0.691 & 0.960 \\
        \cellcolor{green!25}MES[$sat^{cost}$] & 0.524 & 0.831 & 0.470 & 0.806 \\
        \cellcolor{green!25}MES[$sat^{cost}$]-Add1 & 0.766 & 0.934 & 0.775 & 0.917 \\
        \cellcolor{green!25}MES[$sat^{cost}$]-Add1U & 0.780 & 0.938 & 0.797 & 0.927 \\
        \cellcolor{green!25}MES[$sat^{cost}$]-Add1UM & 0.780 & 0.938 & 0.797 & 0.927 \\
        \cellcolor{green!25}MES[$sat^{card}$] & 0.419 & 0.855 & 0.391 & 0.824 \\
        \cellcolor{green!25}MES[$sat^{card}$]-Add1 & 0.710 & 0.951 & 0.697 & 0.955 \\
        \cellcolor{green!25}MES[$sat^{card}$]-Add1U & 0.732 & 0.955 & 0.722 & 0.959 \\
        \cellcolor{green!25}MES[$sat^{card}$]-Add1UM & 0.732 & 0.955 & 0.722 & 0.958 \\\thickhline
        \cellcolor{gray!25}\textsc{LLMRule} & 0.802 & 0.940 & 0.814 & 0.901 \\
        \cellcolor{green!25}MES[$sat^{cost}$]-\textsc{LLMRule} & 0.775 & 0.951 & 0.774 & 0.943 \\
        \cellcolor{green!25}MES[$sat^{card}$]-\textsc{LLMRule} & 0.752 & 0.955 & 0.768 & 0.947 \\
        \thickhline
    \end{tabular}\\
    \medskip
    \begin{tabular}{l|c|c|c|c}
        \thickhline
        PB rule & \makecell{$\omega'[sat^{card}]$\\ (ID)} & \makecell{$\phi$\\ (ID)} & \makecell{$\omega'[sat^{card}]$\\ (OOD)} & \makecell{$\phi$\\ (OOD)} \\\hline
        \cellcolor{gray!25}\textsc{MaxUtil} & 1 & 0.896 & 1 & 0.926 \\
        \cellcolor{gray!25}\textsc{GreedUtil} & 0.964 & 0.962 & 0.984 & 0.967 \\\hline
        \cellcolor{green!25}\textsc{SeqPhrag} & 0.943 & 0.960 & 0.959 & 0.960 \\
        \cellcolor{green!25}\textsc{MaximinSupp} & 0.959 & 0.961 & 0.977 & 0.964 \\
        \cellcolor{green!25}MES[$sat^{cost}$] & 0.763 & 0.854 & 0.732 & 0.820 \\
        \cellcolor{green!25}MES[$sat^{cost}$]-Add1 & 0.942 & 0.942 & 0.945 & 0.925 \\
        \cellcolor{green!25}MES[$sat^{cost}$]-Add1U & 0.961 & 0.948 & 0.970 & 0.936 \\
        \cellcolor{green!25}MES[$sat^{cost}$]-Add1UM & 0.961 & 0.947 & 0.971 & 0.936 \\
        \cellcolor{green!25}MES[$sat^{card}$] & 0.759 & 0.870 & 0.719 & 0.834 \\
        \cellcolor{green!25}MES[$sat^{card}$]-Add1 & 0.939 & 0.956 & 0.957 & 0.959 \\
        \cellcolor{green!25}MES[$sat^{card}$]-Add1U & 0.966 & 0.961 & 0.984 & 0.964 \\
        \cellcolor{green!25}MES[$sat^{card}$]-Add1UM & 0.967 & 0.961 & 0.985 & 0.964 \\\thickhline
        \cellcolor{gray!25}\textsc{LLMRule} & 0.975 & 0.956 & 0.986 & 0.941 \\
        \cellcolor{green!25}MES[$sat^{cost}$]-\textsc{LLMRule} & 0.971 & 0.957 & 0.985 & 0.948 \\
        \cellcolor{green!25}MES[$sat^{card}$]-\textsc{LLMRule} & 0.972 & 0.960 & 0.987 & 0.953 \\
        \thickhline
    \end{tabular}
    \label{tab:res-app_cost}
\end{table}

\subsection{Results}\label{subsec:res}
\paragraph{Approval Ballots.}
Table \ref{tab:res-app_cost} and Figure \ref{fig:tradeoff-app} show results for approval ballots.
In both settings, compared to proportional rules (\textsc{SeqPhrag}, \textsc{MaximinSupp}, and all MES variants), LLMRule yields the best utilitarian social welfare while still maintaining a similar degree of fairness, which is interesting given that the rules generated from LLMRule are greedy in nature. 
Figure \ref{fig:learned-rule1} provides an example of such rule (from one of the three runs of LLMRule) for $sat^{cost}$.
Following the terminology in \cite{rey2023computational}, the rule generated by LLMRule can be formalized as $\{\textsc{Greed}(I,\triangleright)\mid\triangleright \text{ is compatible with }\sfrac{\sqrt{app/n}}{(1+c/b)}\}$. That is, $p$ precedes $p'$ in the order defined by $\triangleright$, denoted as $p\triangleright p'$, if and only if $\sfrac{\sqrt{app(p,\bm{A})/n}}{(1+c(p)/b)}\ge \sfrac{\sqrt{app(p',\bm{A})/n}}{(1+c(p')/b)}$ (recall from Section \ref{subsec:fairness-verify} that $app(p,\bm{A})=|\{i\in\mathcal{N}\mid A_i(p)=1\}|$). 
For comparison, the baseline greedy utilitarian rule, which only concerns utility optimization and is hence significantly less fair, is defined as \textsc{GreedUtil}$(I,\bm{A})=\{\textsc{Greed}(I,\triangleright)\mid\triangleright \text{ is compatible with }app\}$.
Furthermore, the good performance of LLMRule on ID test instances also applies to OOD ones, indicating the generalizability of LLM-generated rules to larger instances with more agents.

When using the same greedy rules from LLMRule as completion methods for MES (MES[$sat$]-\textsc{LLMRule})\footnote{Only MES is non-exhaustive in our considered baselines. We use the exhaustive variant of \textsc{SeqPhrag} and \textsc{MaximinSupp} (originally non-exhaustive) from Pabutools.}, as expected, we observe an improvement in $\phi$ in exchange for a decrease in $\omega$.
Notably, compared to other completion methods (`Add1', `Add1U', `Add1UM'), LLMRule yields the best $\phi$ for $sat=sat^{cost}$ and the best $\omega$ for $sat=sat^{card}$ (please refer to Figure \ref{fig:tradeoff-app-completion} for more succinct plots).

\begin{figure}[t]
    \centering
    \begin{subfigure}[h]{0.47\textwidth}
        \centering
        \footnotesize
        \includegraphics[width=0.99\textwidth]{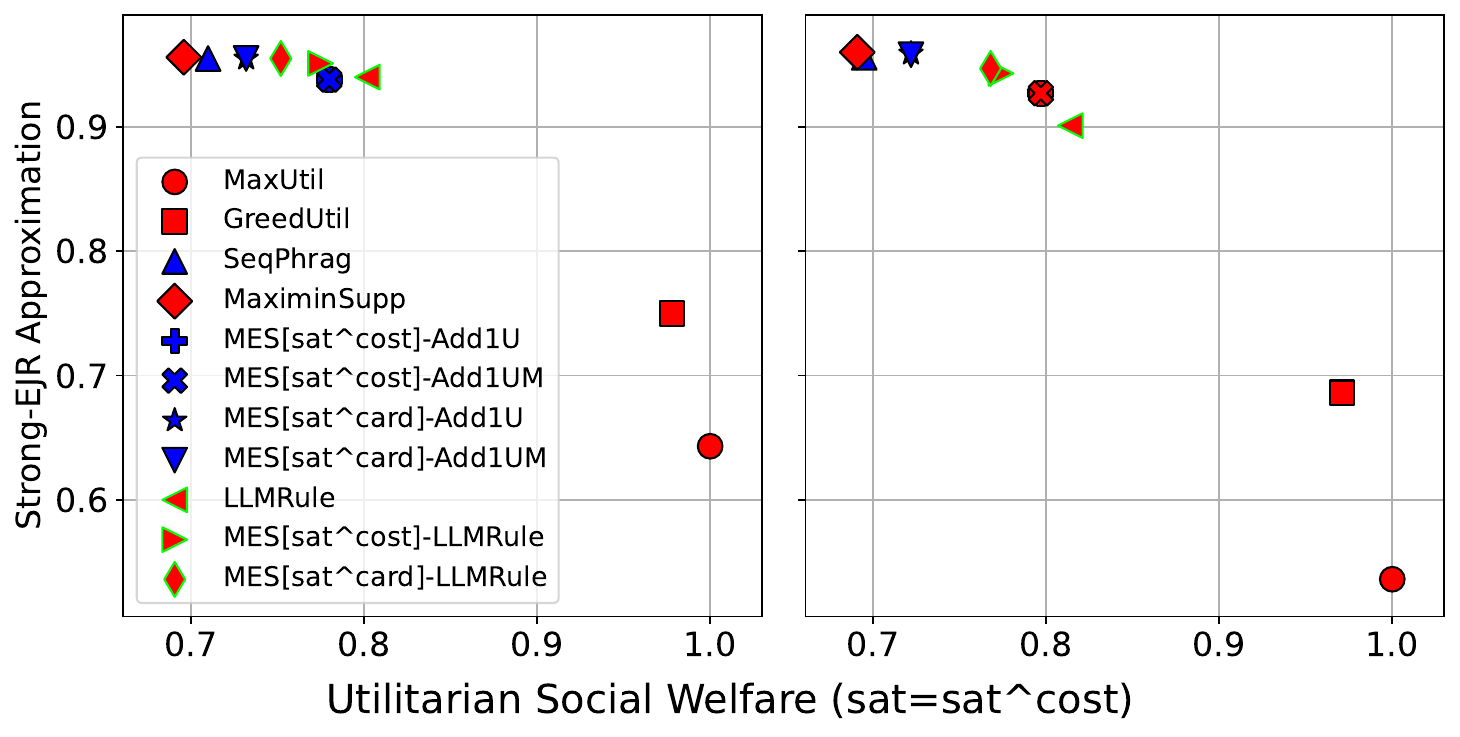}
    \end{subfigure}\\
    \begin{subfigure}[h]{0.47\textwidth}
        \centering
        \footnotesize
        \includegraphics[width=0.99\textwidth]{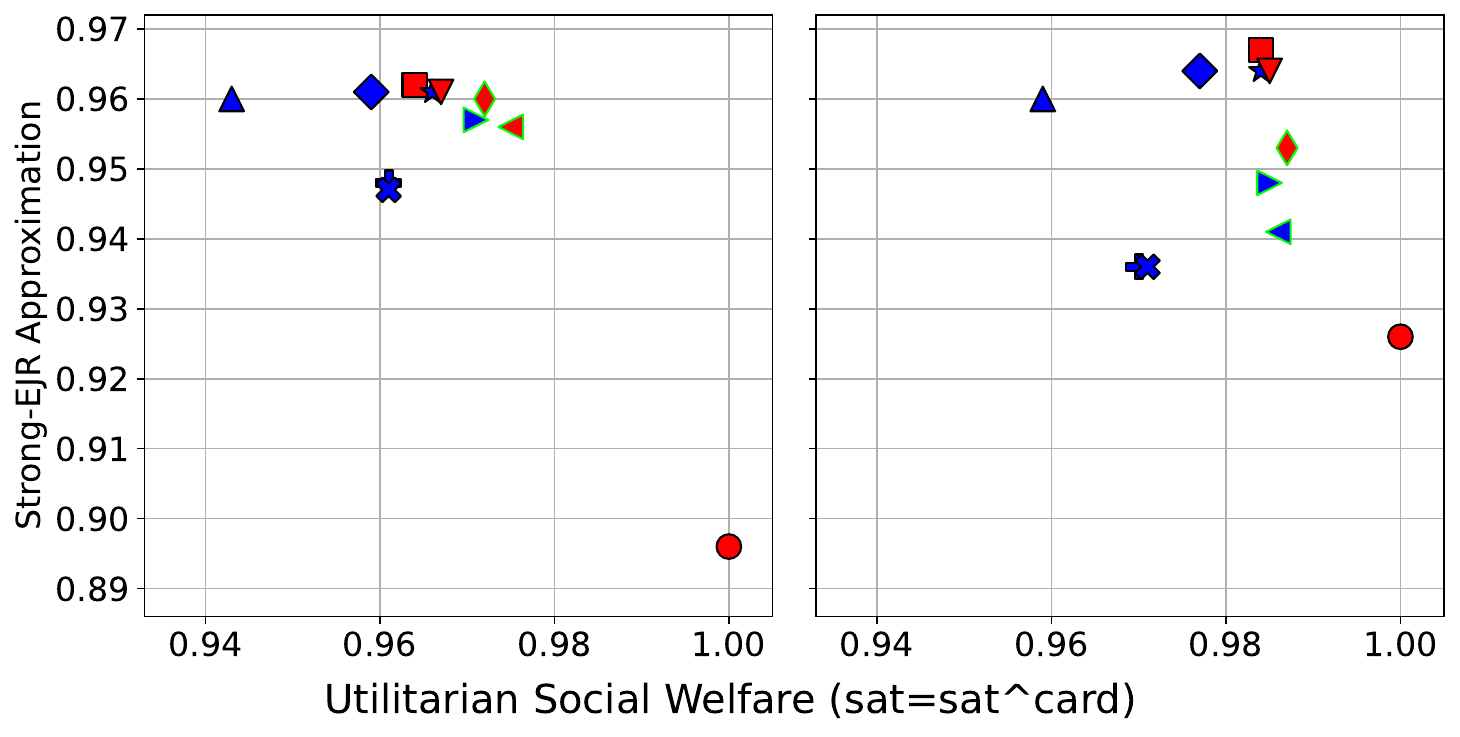}
    \end{subfigure}
    \caption{Utilitarian social welfare vs. Strong-EJR approximation for approval ballots on (left) ID and (right) OOD test sets. The Pareto front is highlighted red. Rules with LLMRule are bordered green. MES and MES-Add1 (either $sat$) were omitted due to being Pareto dominated in all settings.}
    \label{fig:tradeoff-app}
\end{figure}

Overall, the designed rules from LLMRule achieve impressive utility-fairness trade-offs and excel either as standalone greedy utilitarian rules or completion methods for non-exhaustive rules.
Further experiments involving synthetic data in Appendix \ref{app:exp} solidify our claim.

\paragraph{Cardinal Ballots.}
Table \ref{tab:res-cardinal} and Figure \ref{fig:tradeoff-cardinal} show results for the more general setting with cardinal ballots.
Surprisingly, LLMRule yields the best utilitarian social welfare among the baselines (beside the utilitarian maximizer), even surpassing the greedy utilitarian rule.
Meanwhile, the respective Strong-EJR approximation approaches 1, meaning the learned rules are still (empirically) fair.
We provide an example of such rule in Figure \ref{fig:learned-rule3}.
Compared to the previously shown rule for approval ballots, the greedy scheme from this rule appears to be more sophisticated: after computing the scores, it further ensures projects that bring the highest overall utility are prioritized.
Notice that the prompt strategy responsible for producing this rule is newly introduced via our prompt refinement module, which indicates its capability of facilitating the design of novel rules.



\begin{figure}
  \begin{minipage}[h]{.2\textwidth}
    \centering
    \small
    \captionof{table}{Results for cardinal ballots. Utilitarian rules and proportional rules are respectively shaded in gray and green.}
    \begin{tabular}{l|c|c}
        \thickhline
        PB rule & $\omega'$ & $\phi$ \\\hline
        \cellcolor{gray!25}\textsc{MaxUtil} & 1 & 0.921 \\
        \cellcolor{gray!25}\cellcolor{gray!25}\textsc{GreedUtil} & 0.961 & 1 \\\hline
        \cellcolor{green!25}MES & 0.454 & 0.921 \\
        \cellcolor{green!25}MES-Add1 & 0.909 & 1 \\
        \cellcolor{green!25}MES-Add1U & 0.945 & 1 \\
        \cellcolor{green!25}MES-Add1UM & 0.946 & 1 \\\thickhline
        \cellcolor{gray!25}\textsc{LLMRule} & 0.972 & 0.993 \\
        \cellcolor{green!25}MES-\textsc{LLMRule} & 0.970 & 0.995 \\
        \thickhline
    \end{tabular}
    \label{tab:res-cardinal}
  \end{minipage}\hfill
  \begin{minipage}[h]{.25\textwidth}
    \centering
    \includegraphics[width=0.99\textwidth]{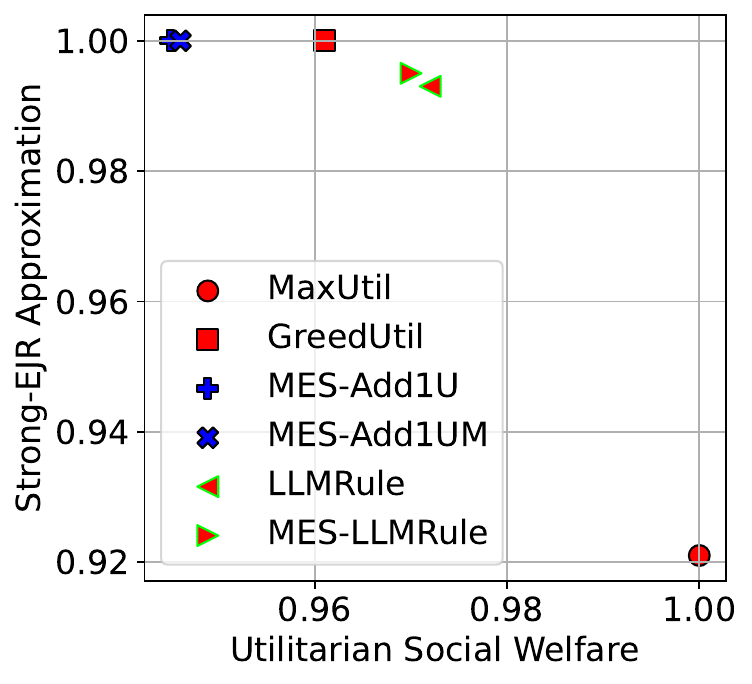}
    \captionof{figure}{$\omega'$ vs. $\phi$ for cardinal ballots. The Pareto front is highlighted red. Rules with LLMRule are bordered green. }
    \label{fig:tradeoff-cardinal}
  \end{minipage}
\end{figure}





\section{Related Work}\label{sec:related}

We discuss the related work focusing on the considered PB settings and LLMs for algorithmic design. 
For a comprehensive review of the PB literature, we refer readers to the surveys \cite{aziz2020participatory,rey2023computational}. For LLMs, please refer to \cite{liu2024systematic}. 
We follow the notations from Sections \ref{sec:prelim} and \ref{sec:method} when applicable.
Appendix \ref{app:related} expands on the quantification of fairness and existing data-driven approaches for designing PB rules.

\paragraph{The Design of Participatory Budgeting Rules.}

As previously mentioned, existing PB rules mainly fall into two types: \emph{utilitarian PB rules} and \emph{proportional PB rules}.
Utilitarian PB rules concern the overall utility of the agents.
For approval ballots, \citet{talmon2019framework} consider the max greedy and proportional greedy PB rules defined with respect to the utilitarian social welfare under the cardinality ($sat^{card}$) and cost ($sat^{cost}$)
satisfaction functions.
While maximizing the utilitarian social welfare under $sat^{card}$ can be done in polynomial time, they show that it is
weakly NP-hard under $sat^{cost}$, for which a pseudo-polynomial time algorithm is provided.
For cardinal ballots, \citet{fluschnik2019fair} consider the problem of utility maximization as a variant of the knapsack problem and propose similar greedy rules based on the knapsack literature \cite{kellerer2004knapsack}.

Proportional PB rules focus on the fairness of the resulting budget allocations with respect to agent preferences.
These rules ensure any sufficiently large and cohesive group of voters should secure allocation outcomes proportional to its size, which is well captured by Strong-EJR.
However, this ideal definition of fairness is not satisfiable in general \cite{chandak2024proportional}.
Therefore, existing works consider weaker but satisfiable fairness definitions, most notably EJR \cite{aziz2017justified,peters2021proportional} and PJR \cite{sanchez2017proportional} as well as their relaxations, e.g., EJR-1 \cite{peters2021proportional}, EJR-X \cite{brill2023proportionality}, PJR-1 \cite{los2022proportional}, PJR-X \cite{brill2023proportionality}.
Unfortunately, it has been theoretically shown that the ``price of fairness'' exists \cite{elkind2022price,fairstein2022welfare}.
More specifically, \cite{fairstein2022welfare} prove that if $f$ satisfies EJR[$sat^{card}$], then for any instance $I$ and profile $\bm{A}$ of approval ballots, the lower bound of $\omega'[sat^{card}](I,\bm{A},f)$ is $\frac{c_{\min}}{n\cdot b}\left\lfloor\frac{b}{c_{\max}}\right\rfloor$, where $c_{\min}=\min_{p\in\mathcal{P}}c(p)$ and $c_{\max}=\max_{p\in\mathcal{P}}c(p)$.
In other words, proportional rules are subject to poor worst-case utilitarian social welfare, which calls for the design of rules that yield better trade-offs between utility and fairness objectives.

\paragraph{Large Language Models for Algorithmic Design.}

LLMs have been applied in diverse tasks, including mathematical reasoning \cite{ahn2024large}, code generation \cite{jiang2024survey}, scientific discovery \cite{wang2023scientific}, and algorithm design \cite{liu2024systematic}. FunSearch \cite{romera2024mathematical} integrated LLMs into an evolutionary search framework for mathematical function design. EoH \cite{liu2024evolution} adopted LLMs to automatically generate, combine and improve heuristics in an evolutionary search. \cite{hu2024automated} proposed ADAS, or Automated Design of Agentic Systems, that aims to automatically generate powerful agentic system designs by using meta agents to program new agents. The significance of integrating LLMs into iterative search frameworks, particularly for challenging design tasks requiring environmental evaluations, has been emphasized in \cite{zhang2024understanding}. Despite these advancements \cite{liu2024systematic}, the application of LLMs to designing PB rules remains unexplored.

Applying existing LLM-based methods, e.g., EoH \cite{liu2024evolution}, requires addressing three challenges as previously stated in our introduction.
Among them, ensuring the verification of the fairness property of interest runs in tractable time is the most important.
Even in the special case of PB (i.e., multiwinner voting, where every project has unit cost and the budget allocation is required to be exhaustive), both EJR and PJR were shown to be coNP-complete to verify given a budget allocation \cite{aziz2017justified,aziz2018complexity}. 
While stronger variants of EJR and PJR may be verifiable in polynomial time, e.g., EJR+ and PJR+ \cite{brill2023robust} (though not as strong as Strong-EJR), this is not the case in general, e.g., as in full justified representation (FJR) and full proportional justified representation (FPJR) \cite{kalayci2025full}.
EJR+ and PJR+ were originally introduced for multiwinner voting settings \cite{brill2023robust} and could be extended to PB \cite{brill2023robust,schmidt2025discrete}. However, 
it is clear that the demonstrated polynomial-time verification procedure of either fairness property (from the proofs of Propositions 2 and 6 in the paper), which hinges on the assumption that every project has unit cost (i.e., each candidate occupies a seat if voted), no longer applies to PB.
In this paper, we circumvent this challenge by devising techniques to enable the fast computation of our newly introduced fairness objective based on our provably equivalent definition of Strong-EJR with reduced verification complexity.

\section{Conclusion}\label{sec:conclu}
In this paper, we propose LLMRule, the first LLM-based framework for automating the design of PB rules.
Through our experiments on various real-world PB instances, we demonstrate that LLMRule can generate rules that perform well in terms of both utility and fairness objectives while requiring minimal domain knowledge and prompt design. 
By this means, LLMRule can significantly reduce manual efforts involved in contemporary research on PB. 






\bibliographystyle{ACM-Reference-Format} 
\bibliography{main}


\clearpage
\onecolumn
\appendix

\renewcommand{\thefigure}{S\arabic{figure}}
\renewcommand{\thetable}{S\arabic{table}}
\renewcommand{\thelisting}{S\arabic{listing}}

\section{Complete Related Work}\label{app:related}

In this appendix, we discuss the related work focusing on the considered PB settings and LLMs for algorithmic design. 
For a comprehensive review of the PB literature, we refer readers to the surveys \cite{aziz2020participatory,rey2023computational}. For LLMs, please refer to \cite{liu2024systematic}. 
We follow the notations from Sections \ref{sec:prelim} and \ref{sec:method} when applicable.

\paragraph{The Design of Participatory Budgeting Rules.}

As mentioned throughout the paper, existing PB rules mainly fall into two types: \emph{utilitarian PB rules} and \emph{proportional PB rules}.
Utilitarian PB rules concern the overall utility of the agents.
For approval ballots, \citet{talmon2019framework} consider the max greedy and proportional greedy PB rules defined with respect to the utilitarian social welfare under the cardinality ($sat^{card}$) and cost ($sat^{cost}$)
satisfaction functions.
While maximizing the utilitarian social welfare under $sat^{card}$ can be done in polynomial time, they show that it is
weakly NP-hard under $sat^{cost}$, for which a pseudo-polynomial time algorithm is provided.
For cardinal ballots, \citet{fluschnik2019fair} consider the problem of utility maximization as a variant of the knapsack problem and propose similar greedy rules based on the knapsack literature \cite{kellerer2004knapsack}.

Proportional PB rules focus on the fairness of the resulting budget allocations with respect to agent preferences.
These rules ensure any sufficiently large and cohesive group of voters should secure allocation outcomes proportional to its size, which is well captured by Strong-EJR.
However, this ideal definition of fairness is not satisfiable in general \cite{chandak2024proportional}.
Therefore, existing works consider weaker but satisfiable fairness definitions, most notably EJR \cite{aziz2017justified,peters2021proportional} and PJR \cite{sanchez2017proportional} as well as their relaxations, e.g., EJR-1 \cite{peters2021proportional}, EJR-X \cite{brill2023proportionality}, PJR-1 \cite{los2022proportional}, PJR-X \cite{brill2023proportionality}.
Unfortunately, it has been theoretically shown that the ``price of fairness'' exists \cite{elkind2022price,fairstein2022welfare}.
More specifically, \cite{fairstein2022welfare} prove that if $f$ satisfies EJR[$sat^{card}$], then for any instance $I$ and profile $\bm{A}$ of approval ballots, the lower bound of $\omega'[sat^{card}](I,\bm{A},f)$ is $\frac{c_{\min}}{n\cdot b}\left\lfloor\frac{b}{c_{\max}}\right\rfloor$, where $c_{\min}=\min_{p\in\mathcal{P}}c(p)$ and $c_{\max}=\max_{p\in\mathcal{P}}c(p)$.
In other words, proportional rules are subject to poor worst-case utilitarian social welfare, which calls for the design of rules that yield better trade-offs between utility and fairness objectives.

\paragraph{Quantification of Fairness.}\label{par:quant-fairness}
In the multiwinner voting literature, the degree of fairness for approval-based committee (ABC) voting rules has been previously studied. \citet{skowron2021proportionality} introduces the so-called \emph{proportionality degree}, which is defined in terms of the average satisfaction yielded by an ABC rule. \citet{tao2025degree} extend \cite{skowron2021proportionality} by defining a new notion, termed \emph{EJR degree}, for quantifying fairness with respect to EJR. It describes the number of represented voters in each cohesive group. For instance, an ABC rule achieving EJR degree 1 satisfies EJR. From this definition, we can see that computing EJR degree is at least as hard as verifying EJR, which prevents this metric from being considered as a fairness objective in our work.


We are also aware of the \emph{representation ratio} metric considered in \cite{fairstein2022welfare,fairstein2024learning} for approval ballots. It is defined as $\sum_{i\in\mathcal{N}}\min\{1,sat_i(\pi)\}/\\\sum_{i\in\mathcal{N}}\min\{1,sat_i(\pi^{\text{OPT}})\}$ where $sat=sat^{card}$. Notice that we can rewrite this ratio as $\sum_{i\in\mathcal{N}}sat_i(\pi)/\sum_{i\in\mathcal{N}}sat_i(\pi^{\text{OPT}})=\omega'[sat^{CC}]$ where $sat=sat^{CC}$, $sat^{CC}(P)=\mathds{1}_{P\neq\emptyset}$ is the Chamberlin-Courant satisfaction function \cite{talmon2019framework}. Therefore, we argue this metric should be treated as a utility objective rather than a fairness objective.

\paragraph{Data-Driven Approaches for Designing PB Rules.}
To our best knowledge, the only work that employs machine learning for PB is \cite{fairstein2024learning}, which attempts to train neural networks for learning PB rules on approval ballots. As with most classical deep learning approaches, the proposed approach suffers from the lack of interpretability and the large sample size required (at least 10,000) for ensuring generalizability. 
Furthermore, only utility objectives are considered\footnote{\cite{fairstein2024learning} define the loss function in terms of the utilitarian social welfare and the representation ratio, the latter of which should be treated as a utility objective as well (see our explanation in \nameref{par:quant-fairness}).}.

\paragraph{Large Language Models for Algorithmic Design.}

LLMs have been applied in diverse tasks, including mathematical reasoning \cite{ahn2024large}, code generation \cite{jiang2024survey}, scientific discovery \cite{wang2023scientific}, and algorithm design \cite{liu2024systematic}. FunSearch \cite{romera2024mathematical} integrated LLMs into an evolutionary search framework for mathematical function design. EoH \cite{liu2024evolution} adopted LLMs to automatically generate, combine and improve heuristics in an evolutionary search. \cite{hu2024automated} proposed ADAS, or Automated Design of Agentic Systems, that aims to automatically generate powerful agentic system designs by using meta agents to program new agents. The significance of integrating LLMs into iterative search frameworks, particularly for challenging design tasks requiring environmental evaluations, has been emphasized in \cite{zhang2024understanding}. Despite these advancements \cite{liu2024systematic}, the application of LLMs to designing PB rules remains unexplored.

Applying existing LLM-based methods, e.g., EoH \cite{liu2024evolution}, requires addressing three challenges as previously stated in our introduction.
Among them, ensuring the verification of the fairness property of interest runs in tractable time is the most important.
Even in the special case of PB (i.e., multiwinner voting, where every project has unit cost and the budget allocation is required to be exhaustive), both EJR and PJR were shown to be coNP-complete to verify given a budget allocation \cite{aziz2017justified,aziz2018complexity}. 
While stronger variants of EJR and PJR may be verifiable in polynomial time, e.g., EJR+ and PJR+ \cite{brill2023robust} (though not as strong as Strong-EJR), this is not the case in general, e.g., as in full justified representation (FJR) and full proportional justified representation (FPJR) \cite{kalayci2025full}.
EJR+ and PJR+ were originally introduced for multiwinner voting settings \cite{brill2023robust} and could be extended to PB \cite{brill2023robust,schmidt2025discrete}. However, 
it is clear that the demonstrated polynomial-time verification procedure of either fairness property (from the proofs of Propositions 2 and 6 in the paper), which hinges on the assumption that every project has unit cost (i.e., each candidate occupies a seat if voted), no longer applies to PB.
In this paper, we circumvent this challenge by devising techniques to enable the fast computation of our newly introduced fairness objective based on our provably equivalent definition of Strong-EJR with reduced verification complexity.



\section{PJR Verification}\label{app:pjr}

An even weaker variant of Strong-EJR is proportional justified representation (PJR) \cite{sanchez2017proportional}, which only requires the cohesive group as a whole to achieve the required satisfaction.
\begin{definition}[PJR for Cardinal Ballots]\label{def:pjr-cardinal}
    Given an instance $I = \langle \mathcal{P}, c, b\rangle$ and a profile $\bm{A}$ of cardinal ballots, a budget allocation $\pi$ is said to satisfy PJR if for all $P \subseteq \mathcal{P}$ and all ($\alpha, P$)-cohesive groups $N$, we have:
    \begin{equation*}
        \sum_{p\in\pi}\max_{i\in N}A_{i}(p)\ge \sum_{p\in P}\min_{i\in N}A_i(p).
    \end{equation*}
\end{definition}
\begin{definition}[PJR for Approval Ballots]\label{def:pjr-app}
    Given an instance $I = \langle \mathcal{P}, c, b\rangle$, a profile $\bm{A}$ of approval ballots, and a satisfaction function $sat$, a budget allocation $\pi$ is said to satisfy PJR[$sat$] if for all $P \subseteq \mathcal{P}$ and all $P$-cohesive groups $N$, we have:
    \begin{equation*}
        sat\left(\bigcup_{i\in N}\{p\in\pi | A_i(p)=1\}\right)\ge sat(P).
    \end{equation*}
\end{definition}

\begin{remark}[Applicability to PJR Verification]\label{rem:pjr}
    Theorem \ref{thm:maximal-group} does not hold when agents in the respective $P$-cohesive groups are aggregated as a single agent. Formally, if $sat\left(\bigcup_{i\in N}\{p\in\pi | A_i(p)=1\}\right)\ge sat(P)$, then for any $N'\subseteq N$, we are not guaranteed to have $sat\left(\bigcup_{i'\in N'}\{p\in\pi | A_{i'}(p)=1\}\right)\ge sat(P)$, particularly when $i\in N\setminus N', N'\subset N$. The same applies to cardinal ballots. Hence, PJR verification does not enjoy similar complexity reduction.
\end{remark}

Remark \ref{rem:pjr} also holds for relaxations of PJR, e.g., PJR up to one project (PJR-1) \cite{los2022proportional} and PJR up to any project (PJR-X) \cite{brill2023proportionality}.


\section{Complete Implementation Details}\label{app:imp}
We consider three problem settings: approval ballots with (i) $sat=sat^{cost}$ and (ii) $sat=sat^{card}$, and (iii) cardinal ballots.
That is, the utility objectives are defined as (i) $\omega[sat^{cost}]$ and (ii) $\omega[sat^{card}]$ (Equation \ref{eq:util-app}), and (iii) $\omega$ (Equation \ref{eq:util-cardinal}).
In each setting, we use LLMRule to design greedy rules,
in which the task for the LLM is to design a priority function for assigning a score to each project. This function takes a list of project costs, the budget constraint, and the ballot profile as input and outputs the scores for the projects.
Using these scores, the allocation outcome $\pi\subseteq\mathcal{P}$ is constructed by greedily adding the projects until none of the remaining projects are affordable given the current budget. 

We conduct three runs of LLMRule for each problem setting and report the average performance. 
The LLM temperature is fixed at 1.
The evaluation resource $t$ is defined in terms of the maximum number of populations, for which we set to 20. 
The population size $h$ is 20, and the maximum running time for each instance is 60 seconds. 
When computing $\phi$ (Equations \ref{eq:strong-ejr-approx-card} and \ref{eq:strong-ejr-approx}), we consider the top-$\sigma$ maximal cohesive groups (Section \ref{subsec:fairness-verify}) with $\sigma$ capped at 100 if needed.
Unless otherwise stated, we assign the threshold $\varepsilon$ defined in Equation \ref{eq:fitness-penalty} as the largest Strong-EJR approximation returned by the baselines (averaged across all training instances), which is typically from MES-Add1U for cardinal ballots and MES[$sat^{card}$]-Add1U for approval ballots. 

For the proposed prompt refinement module (Section \ref{subsec:fairness-eval}), we set the size of top individuals in the population, $l$, to 3; the number of maximum consecutive populations without improvement (for detecting stagnation), $t$, to 3; the size of the top individuals produced by a prompt strategy (for computing its score), $d$, to 3; and the maximum size of the set of prompt strategies to 5.
For quality control of the LLM-generated prompt strategies, we only accept one if at least one of its produced rules is valid, i.e., not experiencing timeout or coding errors.

All above choices were made mainly to ensure reasonable runtime (typically half an hour) without hampering the generation of nontrivial rules.

\section{Prompts}\label{app:prompts}

\subsection{Prompt Tasks}

Table \ref{tab:prompt-tasks} lists the task descriptions used in LLM prompts.

\begin{table*}[h]
    \centering
    \caption{Task descriptions used in LLM prompts. }
    \begin{tabular}{p{0.16\textwidth}|p{0.8\textwidth}}
        \hline
        \multicolumn{1}{c|}{Ballot type} & \multicolumn{1}{c}{Task description} \\\hline
        Approval ($sat^{cost}$) & Given a set of M projects with varied construction costs, a budget, and a set of N approval samples on whether each project should be constructed or not, I need help finding a priority score for each project in order to maximize the average total cost of the selected and approved projects across all samples. \\\hline
        Approval ($sat^{card}$) & Given a set of M projects with varied construction costs, a budget, and a set of N approval samples on whether each project should be constructed or not, I need help finding a priority score for each project in order to maximize the average number of the selected and approved projects across all samples. \\\hline
        Cardinal & Given a set of M projects with varied construction costs, a budget, and a set of N cardinal samples each of which assigns a valuation to each project, I need help finding a priority score for each project in order to maximize the average utility value across all samples. \\
        \hline
    \end{tabular}
    \label{tab:prompt-tasks}
\end{table*}

\clearpage
\subsection{Variation Prompts}\label{subapp:var-prompt}
Figure \ref{fig:init_prompt-design} illustrates the exact initialization prompt and the initial prompts for exploration and modification.

\begin{figure*}[h]
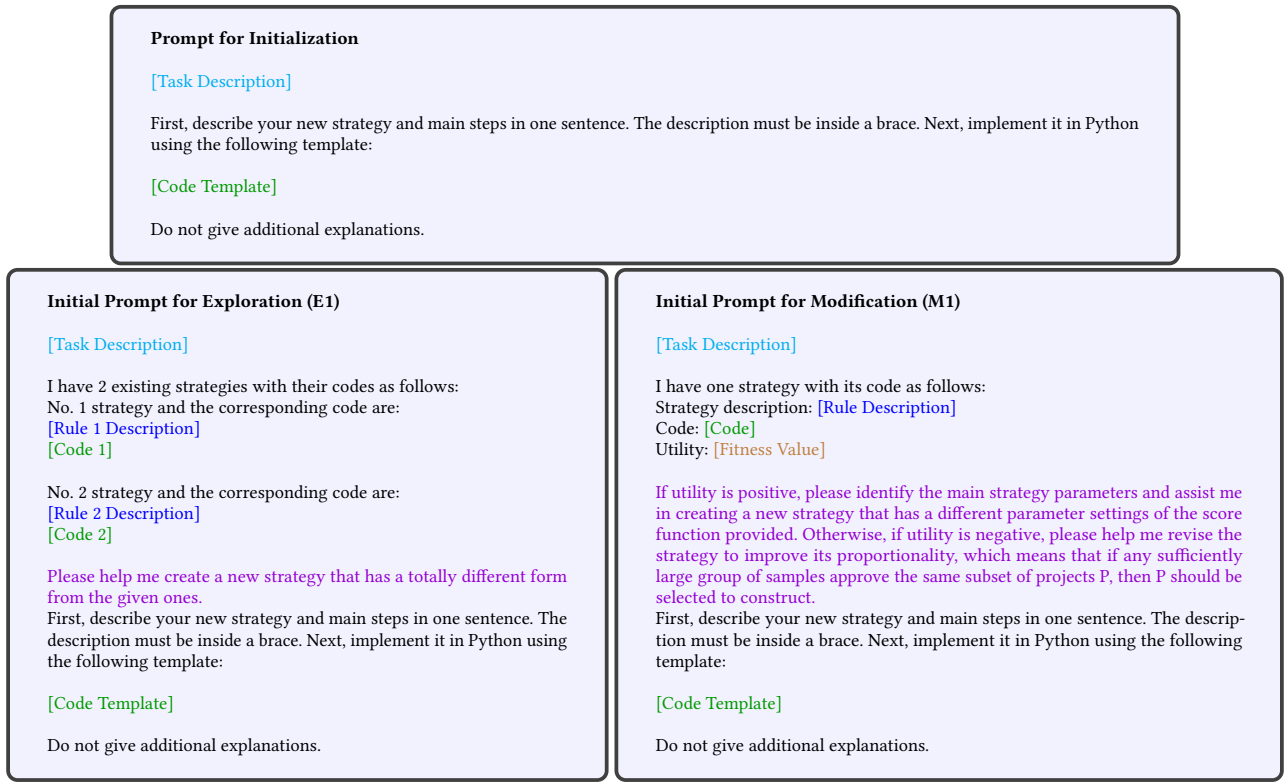

    \centering
    \begin{subfigure}[h]{0.8\textwidth}
        \centering
        \footnotesize
        \begin{tcolorbox}[width=\textwidth, colback=blue!5]
        \textbf{Prompt for Initialization}

        \hphantom\\

        \textcolor{cyan}{[Task Description]}

        \hphantom\\

        First, describe your new strategy and main steps in one sentence. The description must be inside a brace. Next, implement it in Python using the following template:

        \hphantom\\

        \textcolor{dkgreen}{[Code Template]}

        \hphantom\\
        
        Do not give additional explanations.
        \end{tcolorbox}
    \end{subfigure}\\
    \begin{subfigure}[h]{0.45\textwidth}
        \centering
        \footnotesize
        \begin{tcolorbox}[width=\textwidth, colback=blue!5]
        \textbf{Initial Prompt for Exploration (E1)}

        \hphantom\\

        \textcolor{cyan}{[Task Description]}

        \hphantom\\

        I have 2 existing strategies with their codes as follows:\\
        No. 1 strategy and the corresponding code are:\\
        \textcolor{blue}{[Rule 1 Description]}\\
        \textcolor{dkgreen}{[Code 1]}

        \hphantom\\

        No. 2 strategy and the corresponding code are:\\
        \textcolor{blue}{[Rule 2 Description]}\\
        \textcolor{dkgreen}{[Code 2]}

        \hphantom\\

        \textcolor{mauve}{Please help me create a new strategy that has a totally different form from the given ones.}\\
        First, describe your new strategy and main steps in one sentence. The description must be inside a brace. Next, implement it in Python using the following template:

        \hphantom\\

        \textcolor{dkgreen}{[Code Template]}

        \hphantom\\
        
        Do not give additional explanations.
        \end{tcolorbox}
    \end{subfigure}%
    ~
    \begin{subfigure}[h]{0.5\textwidth}
        \centering
        \footnotesize
        \begin{tcolorbox}[width=\textwidth, colback=blue!5]
        \textbf{Initial Prompt for Modification (M1)}

        \hphantom\\

        \textcolor{cyan}{[Task Description]}

        \hphantom\\

        I have one strategy with its code as follows:\\
        Strategy description:  \textcolor{blue}{[Rule Description]}\\
        Code: \textcolor{dkgreen}{[Code]}\\
        Utility: \textcolor{brown}{[Fitness Value]}

        \hphantom\\

        \textcolor{mauve}{If utility is positive, please identify the main strategy parameters and assist me in creating a new strategy that has a different parameter settings of the score function provided. Otherwise, if utility is negative, please help me revise the strategy to improve its proportionality, which means that if any sufficiently large group of samples approve the same subset of projects P, then P should be selected to construct.}\\
        First, describe your new strategy and main steps in one sentence. The description must be inside a brace. Next, implement it in Python using the following template:

        \hphantom\\

        \textcolor{dkgreen}{[Code Template]}

        \hphantom\\
        
        Do not give additional explanations.
        \end{tcolorbox}
    \end{subfigure}
    \caption{Prompts used for initialization, exploration, and modification. The prompt strategies are marked in purple. See Table \ref{tab:prompt-tasks} for ``Task Description'' and Listings \ref{lst:code-template-app} and \ref{lst:code-template-cardinal} for the designed ``Code Template''.}
    \label{fig:init_prompt-design}
\end{figure*}

\clearpage
\subsection{Prompt Evolution}\label{subapp:prompt-evol}
Figure \ref{fig:evol_prompt-design} illustrates the exact prompt used for evolving prompt strategies.

\begin{figure*}[h]
    \centering
    \begin{subfigure}[h]{0.95\textwidth}
        \centering
        \scriptsize
        \begin{tcolorbox}[width=\textwidth, colback=blue!5]
        \textbf{Prompt for Evolving Exploration Prompt Strategies}

        \hphantom\\

        \textcolor{cyan}{[Task Description]}

        \hphantom\\

        I want to leverage the capabilities of LLMs to generate heuristic algorithms that can efficiently tackle this problem. I have already developed a set of initial prompts and observed the corresponding outputs. However, to improve the effectiveness of these algorithms, we need your assistance in carefully analyzing the existing prompts and their results. Based on this analysis, we ask you to generate new prompts that will help us achieve better results in solving the problem.

        \hphantom\\

        I have X existing prompts with average score (the higher the better) as follows:\\
        No. 1 prompt:\\
        Content: \textcolor{mauve}{[Prompt Strategy]}\\
        Score: \textcolor{orange}{[Prompt Fitness Value]}\\
        ...\\
        No. X prompt:\\
        Content: \textcolor{mauve}{[Prompt Strategy]}\\
        Score: \textcolor{orange}{[Prompt Fitness Value]}

        \hphantom\\

        Note that we categorize prompts into two groups: Exploration and Modification. Those I just showed are Exploration prompts, which ask LLMs to generate new strategies that are as different as possible from the input strategies.

        \hphantom\\

        Please help me create a new Exploration prompt that has a totally different form from the given ones but can be motivated from them. Describe your new prompt and main steps in one sentence. The description must be inside a brace. Do not give additional explanations.
        \end{tcolorbox}
    \end{subfigure}\\
    \begin{subfigure}[h]{0.95\textwidth}
        \centering
        \scriptsize
        \begin{tcolorbox}[width=\textwidth, colback=blue!5]
        \textbf{Prompt for Evolving Modification Prompt Strategies}

        \hphantom\\

        \textcolor{cyan}{[Task Description]}

        \hphantom\\

        I want to leverage the capabilities of LLMs to generate heuristic algorithms that can efficiently tackle this problem. I have already developed a set of initial prompts and observed the corresponding outputs. However, to improve the effectiveness of these algorithms, we need your assistance in carefully analyzing the existing prompts and their results. Based on this analysis, we ask you to generate new prompts that will help us achieve better results in solving the problem.

        \hphantom\\

        I have X existing prompts with average score (the higher the better) as follows:\\
        No. 1 prompt:\\
        Content: \textcolor{mauve}{[Prompt Strategy]}\\
        Score: \textcolor{orange}{[Prompt Fitness Value]}\\
        ...\\
        No. X prompt:\\
        Content: \textcolor{mauve}{[Prompt Strategy]}\\
        Score: \textcolor{orange}{[Prompt Fitness Value]}

        \hphantom\\

        Note that we categorize prompts into two groups: Exploration and Modification. Those I just showed are Modification prompts, which ask LLMs to refine an input strategy by modifying, adjusting parameters, and removing redundant parts.

        \hphantom\\

        Please help me create a new Modification prompt that has a totally different form from the given ones but can be motivated from them. Describe your new prompt and main steps in one sentence. The description must be inside a brace. Do not give additional explanations.
        \end{tcolorbox}
    \end{subfigure}
    \caption{Prompts used for evolving prompt strategies, where X is the size of the current set of prompt strategies.}
    \label{fig:evol_prompt-design}
\end{figure*}

\clearpage
\section{Code Template}\label{app:code-template}

Listings \ref{lst:code-template-app} and \ref{lst:code-template-cardinal} respectively show the code templates for instructing the LLM to follow when designing rules for approval and cardinal ballots.

\begin{listing*}[h]%
\caption{Code template for approval ballots.}%
\label{lst:code-template-app}%
\begin{lstlisting}[language=Python]
import numpy as np

def priority(project_costs, budget, approval_mat):
    '''
    Args:
    project_costs (np.ndarray): Array of length M storing the project costs.
    budget (int or float): A positive value.
    approval_mat (np.ndarray): An N-by-M matrix where each row contains a sample of binary approvals for the respective projects.

    Returns:
    scores (np.ndarray): Array of priority scores for the projects.
    '''

    # Placeholder (replace with your actual implementation)
    scores = ...

    return scores
\end{lstlisting}
\end{listing*}

\begin{listing*}[h]%
\caption{Code template for cardinal ballots.}%
\label{lst:code-template-cardinal}%
\begin{lstlisting}[language=Python]
import numpy as np

def priority(project_costs, budget, valuation_mat):
    '''
    Args:
    project_costs (np.ndarray): Array of length M storing the project costs.
    budget (int or float): A positive value.
    valuation_mat (np.ndarray): An N-by-M matrix where each row contains a sample of valuations for the respective projects.

    Returns:
    scores (np.ndarray): Array of priority scores for the projects.
    '''

    # Placeholder (replace with your actual implementation)
    scores = ...

    return scores
\end{lstlisting}
\end{listing*}











\clearpage
\section{Designed Rules}\label{app:designed-rules}

Figures \ref{fig:learned-rule2} and \ref{fig:learned-rule3} respectively show examples of the learned rules from LLMRule for approval ($sat=sat^{card}$) and cardinal ballots.

\begin{figure*}[h]
    \footnotesize
    \centering
    \begin{tcolorbox}[width=0.65\textwidth, colback=blue!5]
        \textbf{Rule Description (generated using the initial M1 prompt strategy)}

        \hphantom\\

        The new strategy modifies the scoring function to emphasize projects with high approval rates while proportionally reducing the influence of costly projects, using a linear budget fraction combined with a logarithmic transformation of approval rates.
    \end{tcolorbox}
    \medskip
\begin{lstlisting}[language=Python]
import numpy as np

def priority(project_costs, budget, approval_mat):
    '''
    Args:
    project_costs (np.ndarray): Array of length M storing the project costs.
    budget (int or float): A positive value.
    approval_mat (np.ndarray): An N-by-M matrix where each row contains a sample of binary approvals for the respective projects.

    Returns:
    scores (np.ndarray): Array of priority scores for the projects.
    '''

    # Calculate the average approval rates
    avg_approvals = np.mean(approval_mat, axis=0)

    # Calculate budget fractions
    budget_fractions = budget / project_costs

    # Calculate scores based on linear budget fractions and logarithmic approval rates
    scores = avg_approvals * (budget_fractions) * np.log(1 + avg_approvals)

    return scores
\end{lstlisting}
    \caption{High-level description of an LLM-generated PB rule for approval ballots with $sat=sat^{card}$ and its associated code.}
    \label{fig:learned-rule2}
\end{figure*}

\begin{figure*}[h]
    \footnotesize
    \centering
    \begin{tcolorbox}[width=0.65\textwidth, colback=blue!5]
        \textbf{LLM-Generated Prompt Strategy `MN1'}

        \hphantom\\

        Given a set of project valuations and current selection criteria, please devise a method to balance the utility across diverse project priorities while ensuring that at least the top N projects, based on sample approval, are included in the final selection.
    \end{tcolorbox}
    \medskip
    \begin{tcolorbox}[width=0.65\textwidth, colback=blue!5]
        \textbf{Rule Description (generated using the above `MN1' prompt strategy)}

        \hphantom\\

        The new strategy computes priority scores by calculating a weighted combination of the average valuation, a budget-based cost-effectiveness ratio, and a penalty for normalized cost, while ensuring that at least the top N projects with the highest valuations are included in the final selection.
    \end{tcolorbox}
    \medskip
\begin{lstlisting}[language=Python]
import numpy as np

def priority(project_costs, budget, valuation_mat):
    '''
    Args:
    project_costs (np.ndarray): Array of length M storing the project costs.
    budget (int or float): A positive value.
    valuation_mat (np.ndarray): An N-by-M matrix where each row contains a sample of valuations for the respective projects.

    Returns:
    scores (np.ndarray): Array of priority scores for the projects.
    '''

    M = len(project_costs)
    N = valuation_mat.shape[0]
    scores = np.zeros(M)

    for i in range(M):
        average_valuation = np.mean(valuation_mat[:, i])
        cost_efficiency = (budget / project_costs[i]) if project_costs[i] <= budget else 0
        normalized_cost_penalty = 1 - (project_costs[i] / budget) if project_costs[i] <= budget else 0

        # Calculate score
        scores[i] = (average_valuation * cost_efficiency) - normalized_cost_penalty

    # Ensure top N projects based on average valuation are included
    top_n_indices = np.argsort(-np.mean(valuation_mat, axis=0))[:N]
    for idx in top_n_indices:
        if scores[idx] < 0:
            scores[idx] = 0  # Ensure the score is non-negative

    return scores
\end{lstlisting}
    \caption{High-level description of an LLM-generated PB rule for cardinal ballots and its associated code.}
    \label{fig:learned-rule3}
\end{figure*}

\clearpage
\paragraph{LLMRule as Completion Methods.}

Figure \ref{fig:tradeoff-app-completion} (truncated from Figure \ref{fig:tradeoff-app}) highlights the performance of LLMRule when utilized as completion methods for MES.

\begin{figure}[h]
    \centering
    \begin{subfigure}[h]{0.65\textwidth}
        \centering
        \footnotesize
        \includegraphics[width=0.99\textwidth]{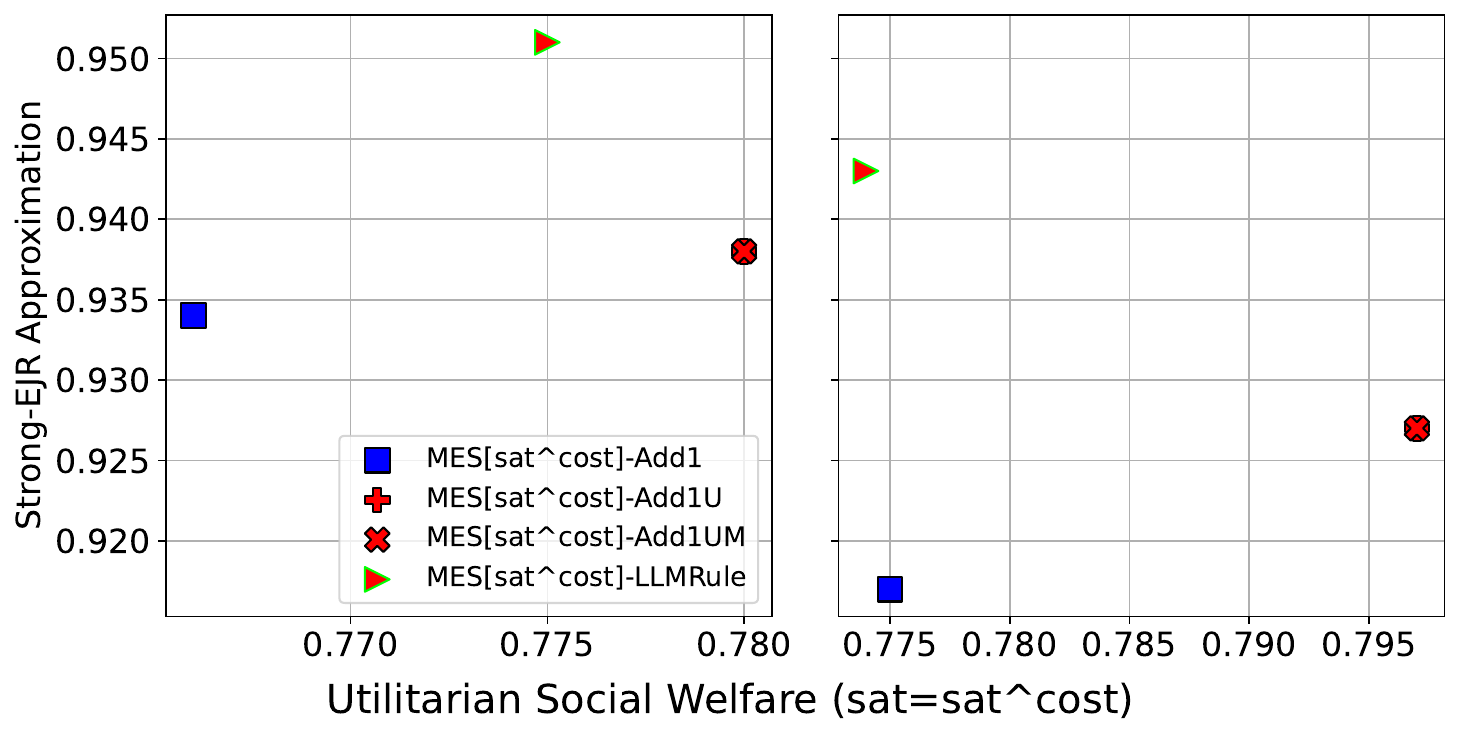}
    \end{subfigure}\\
    \begin{subfigure}[h]{0.65\textwidth}
        \centering
        \footnotesize
        \includegraphics[width=0.99\textwidth]{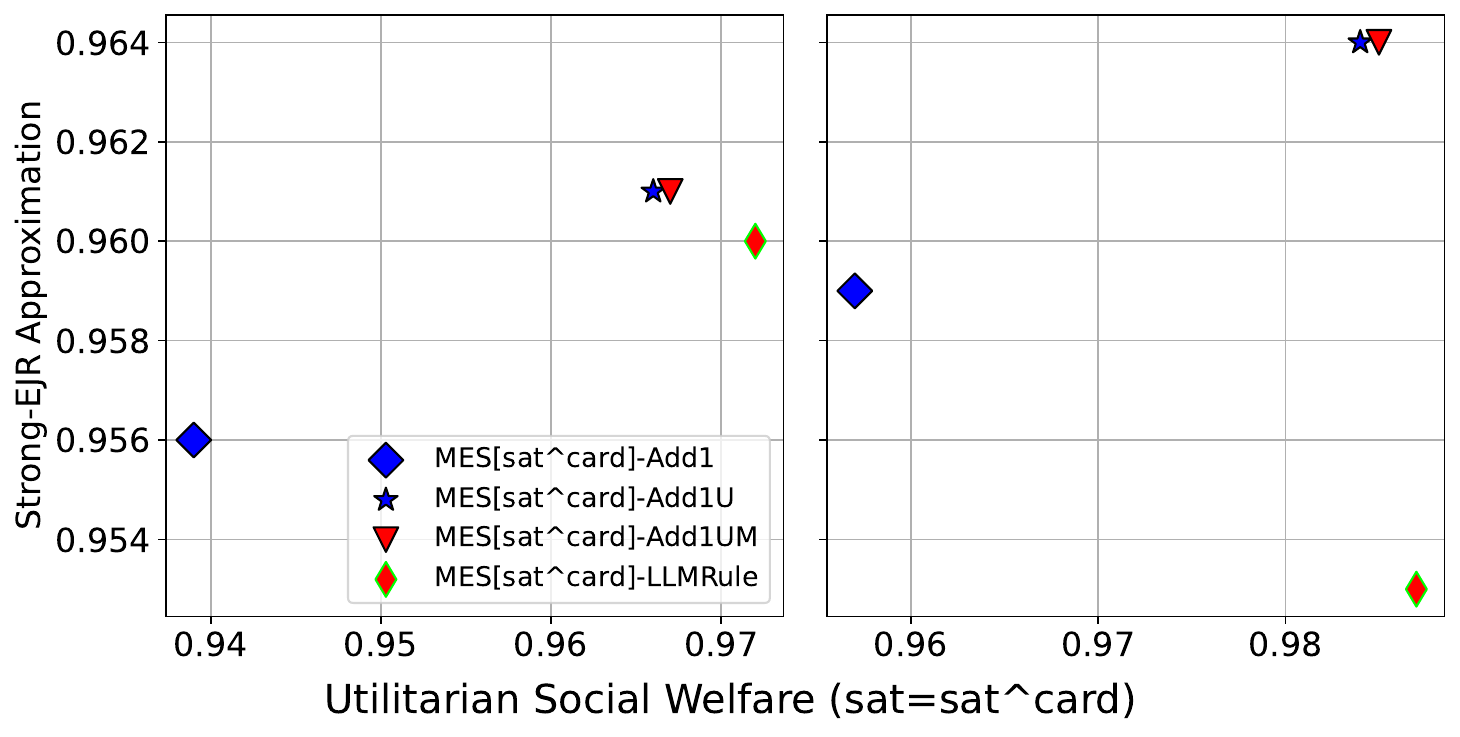}
    \end{subfigure}
    \caption{Utilitarian social welfare vs. Strong-EJR approximation of various completion methods for approval ballots on (left) ID and (right) OOD test sets. The Pareto front is highlighted red.}
    \label{fig:tradeoff-app-completion}
\end{figure}



    


    

    

\clearpage
\section{Additional Results with Synthetic Data}\label{app:exp}
We adopt the setup from \cite{fairstein2024learning}, which employs the standard two-dimensional Euclidean model \cite{enelow1984spatial,elkind2017multiwinner,boehmer2024guide}, to generate synthetic PB instances.
Table \ref{tab:data-syn} provides the specifications of our datasets.
Using similar experimental setups, we train LLMRule on synthetic PB instances and test the best learned rules on a hold-out set of synthetic PB instances as well as the previously employed Pabulib (small and large) test sets.
Table \ref{tab:res-app_cost-syn} and Figure \ref{fig:tradeoff-app-syn} summarize our results.


\begin{table}[h]
    \centering
    \caption{Synthetic data specifications.} 
    \begin{tabular}{l|c|c|c}
        Split & Sample size & $n$ & $m$ \\\hline
        Train & 100 & $[100,1000]$ & $[10,25]$ \\
        Test & 1,000 & $[100,2000]$ & $[10,25]$ \\
    \end{tabular}
    \label{tab:data-syn}
\end{table}

\begin{table}[h]
    \centering
    \caption{Results for approval ballots with (top) $sat=sat^{cost}$ and (bottom) $sat=sat^{card}$ on synthetic and Pabulib test sets. Utilitarian rules and proportional rules are respectively shaded in gray and green.} 
    \begin{tabular}{l|c|c|c|c|c|c}
        \thickhline
        PB rule & \makecell{$\omega'[sat^{cost}]$\\ (synthetic)} & \makecell{$\phi$\\ (synthetic)} & \makecell{$\omega'[sat^{cost}]$\\ (Pabulib-small)} & \makecell{$\phi$\\ (Pabulib-small)} & \makecell{$\omega'[sat^{cost}]$\\ (Pabulib-large)} & \makecell{$\phi$\\ (Pabulib-large)} \\\hline
        \cellcolor{gray!25}\textsc{MaxUtil} & 1 & 0.537 & 1 & 0.643 & 1 & 0.536 \\
        \cellcolor{gray!25}\textsc{GreedUtil} & 0.971 & 0.606 & 0.978 & 0.750 & 0.971 & 0.686 \\\hline
        \cellcolor{green!25}\textsc{SeqPhrag} & 0.766 & 0.775 & 0.710 & 0.955 & 0.695 & 0.956 \\
        \cellcolor{green!25}\textsc{MaximinSupp} & 0.734 & 0.771 & 0.696 & 0.956 & 0.691 & 0.960 \\
        \cellcolor{green!25}MES[$sat^{cost}$] & 0.657 & 0.617 & 0.524 & 0.831 & 0.470 & 0.806 \\
        \cellcolor{green!25}MES[$sat^{cost}$]-Add1 & 0.882 & 0.691 & 0.766 & 0.934 & 0.775 & 0.917 \\
        \cellcolor{green!25}MES[$sat^{cost}$]-Add1U & 0.895 & 0.702 & 0.780 & 0.938 & 0.797 & 0.927 \\
        \cellcolor{green!25}MES[$sat^{cost}$]-Add1UM & 0.895 & 0.702 & 0.780 & 0.938 & 0.797 & 0.927 \\
        \cellcolor{green!25}MES[$sat^{card}$] & 0.528 & 0.658 & 0.419 & 0.855 & 0.391 & 0.824 \\
        \cellcolor{green!25}MES[$sat^{card}$]-Add1 & 0.770 & 0.773 & 0.710 & 0.951 & 0.697 & 0.955 \\
        \cellcolor{green!25}MES[$sat^{card}$]-Add1U & 0.796 & 0.783 & 0.732 & 0.955 & 0.722 & 0.959 \\
        \cellcolor{green!25}MES[$sat^{card}$]-Add1UM & 0.797 & 0.782 & 0.732 & 0.955 & 0.722 & 0.958 \\\thickhline
        \cellcolor{gray!25}\textsc{LLMRule} & 0.894 & 0.691 & 0.884 & 0.863 & 0.848 & 0.848 \\
        \cellcolor{green!25}MES[$sat^{cost}$]-\textsc{LLMRule} & 0.873 & 0.753 & 0.762 & 0.952 & 0.747 & 0.947 \\
        \cellcolor{green!25}MES[$sat^{card}$]-\textsc{LLMRule} & 0.830 & 0.775 & 0.737 & 0.956 & 0.744 & 0.949 \\
        \thickhline
    \end{tabular}\\
    \medskip
    \begin{tabular}{l|c|c|c|c|c|c}
        \thickhline
        PB rule & \makecell{$\omega'[sat^{card}]$\\ (synthetic)} & \makecell{$\phi$\\ (synthetic)} & \makecell{$\omega'[sat^{card}]$\\ (Pabulib-small)} & \makecell{$\phi$\\ (Pabulib-small)} & \makecell{$\omega'[sat^{card}]$\\ (Pabulib-large)} & \makecell{$\phi$\\ (Pabulib-large)} \\\hline
        \cellcolor{gray!25}\textsc{MaxUtil} & 1 & 0.793 & 1 & 0.896 & 1 & 0.926 \\
        \cellcolor{gray!25}\textsc{GreedUtil} & 0.981 & 0.807 & 0.964 & 0.962 & 0.984 & 0.967 \\\hline
        \cellcolor{green!25}\textsc{SeqPhrag} & 0.952 & 0.795 & 0.943 & 0.960 & 0.959 & 0.960 \\
        \cellcolor{green!25}\textsc{MaximinSupp} & 0.951 & 0.794 & 0.959 & 0.961 & 0.977 & 0.964 \\
        \cellcolor{green!25}MES[$sat^{cost}$] & 0.739 & 0.647 & 0.763 & 0.854 & 0.732 & 0.820 \\
        \cellcolor{green!25}MES[$sat^{cost}$]-Add1 & 0.906 & 0.715 & 0.942 & 0.942 & 0.945 & 0.925 \\
        \cellcolor{green!25}MES[$sat^{cost}$]-Add1U & 0.922 & 0.726 & 0.961 & 0.948 & 0.970 & 0.936 \\
        \cellcolor{green!25}MES[$sat^{cost}$]-Add1UM & 0.922 & 0.725 & 0.961 & 0.947 & 0.971 & 0.936 \\
        \cellcolor{green!25}MES[$sat^{card}$] & 0.759 & 0.688 & 0.759 & 0.870 & 0.719 & 0.834 \\
        \cellcolor{green!25}MES[$sat^{card}$]-Add1 & 0.952 & 0.793 & 0.939 & 0.956 & 0.957 & 0.959 \\
        \cellcolor{green!25}MES[$sat^{card}$]-Add1U & 0.977 & 0.803 & 0.966 & 0.961 & 0.984 & 0.964 \\
        \cellcolor{green!25}MES[$sat^{card}$]-Add1UM & 0.977 & 0.802 & 0.967 & 0.961 & 0.985 & 0.964 \\\thickhline
        \cellcolor{gray!25}\textsc{LLMRule} & 0.967 & 0.790 & 0.973 & 0.958 & 0.986 & 0.957 \\
        \cellcolor{green!25}MES[$sat^{cost}$]-\textsc{LLMRule} & 0.961 & 0.779 & 0.970 & 0.955 & 0.986 & 0.957 \\
        \cellcolor{green!25}MES[$sat^{card}$]-\textsc{LLMRule} & 0.985 & 0.804 & 0.970 & 0.959 & 0.986 & 0.959 \\
        \thickhline
    \end{tabular}
    \label{tab:res-app_cost-syn}
\end{table}

\begin{figure}[h]
    \centering
    \begin{subfigure}[h]{0.98\textwidth}
        \centering
        \footnotesize
        \includegraphics[width=0.99\textwidth]{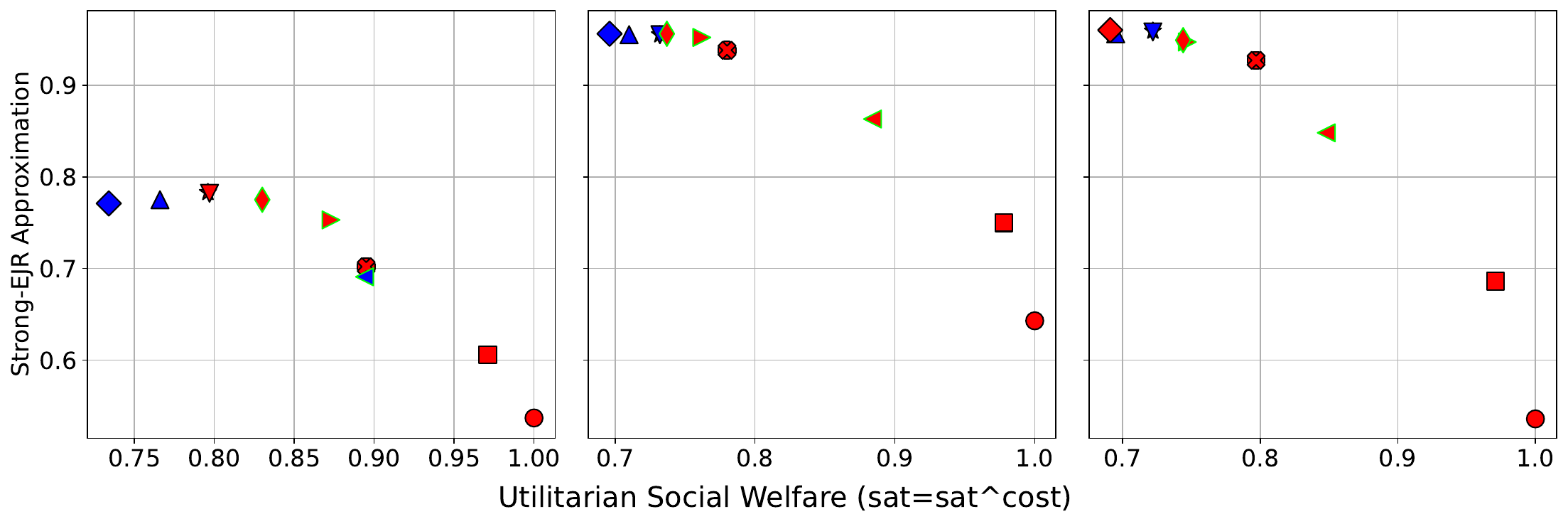}
    \end{subfigure}\\
    \begin{subfigure}[h]{0.98\textwidth}
        \centering
        \footnotesize
        \includegraphics[width=0.99\textwidth]{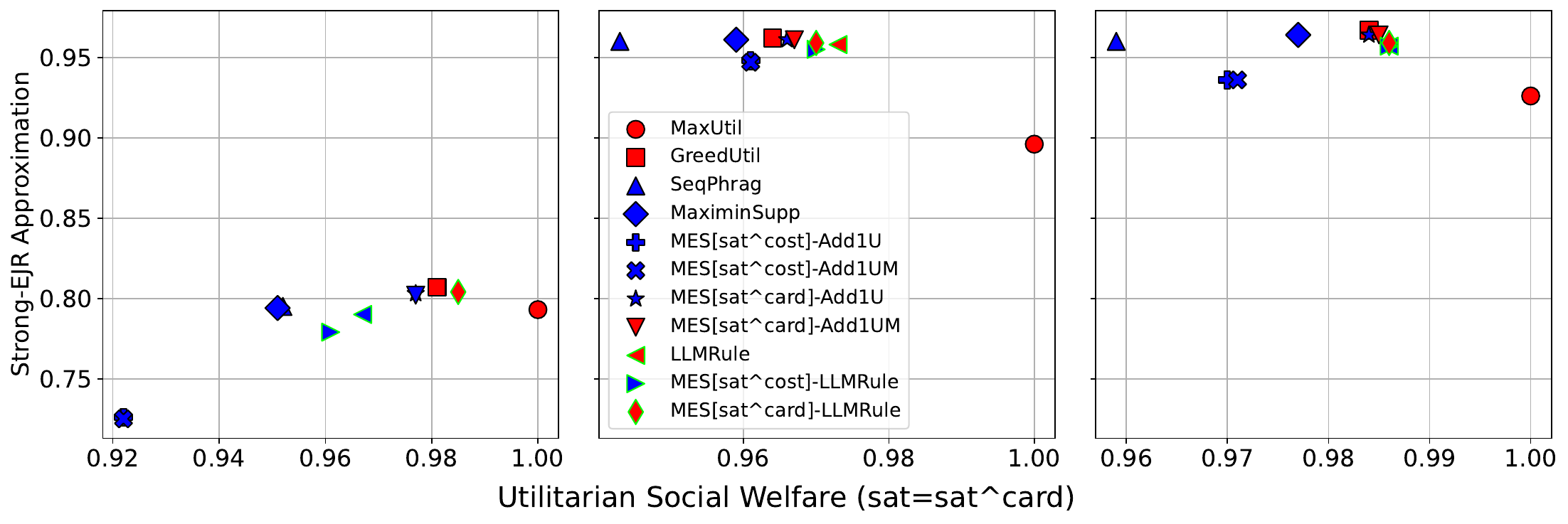}
    \end{subfigure}
    \caption{Utilitarian social welfare vs. Strong-EJR approximation for approval ballots on (left) synthetic, (center) Pabulib-small and (right) Pabulib-large test sets. The Pareto front is highlighted red. Rules with LLMRule are bordered green.}
    \label{fig:tradeoff-app-syn}
\end{figure}

\begin{figure}[h]
    \centering
    \begin{subfigure}[h]{0.98\textwidth}
        \centering
        \footnotesize
        \includegraphics[width=0.99\textwidth]{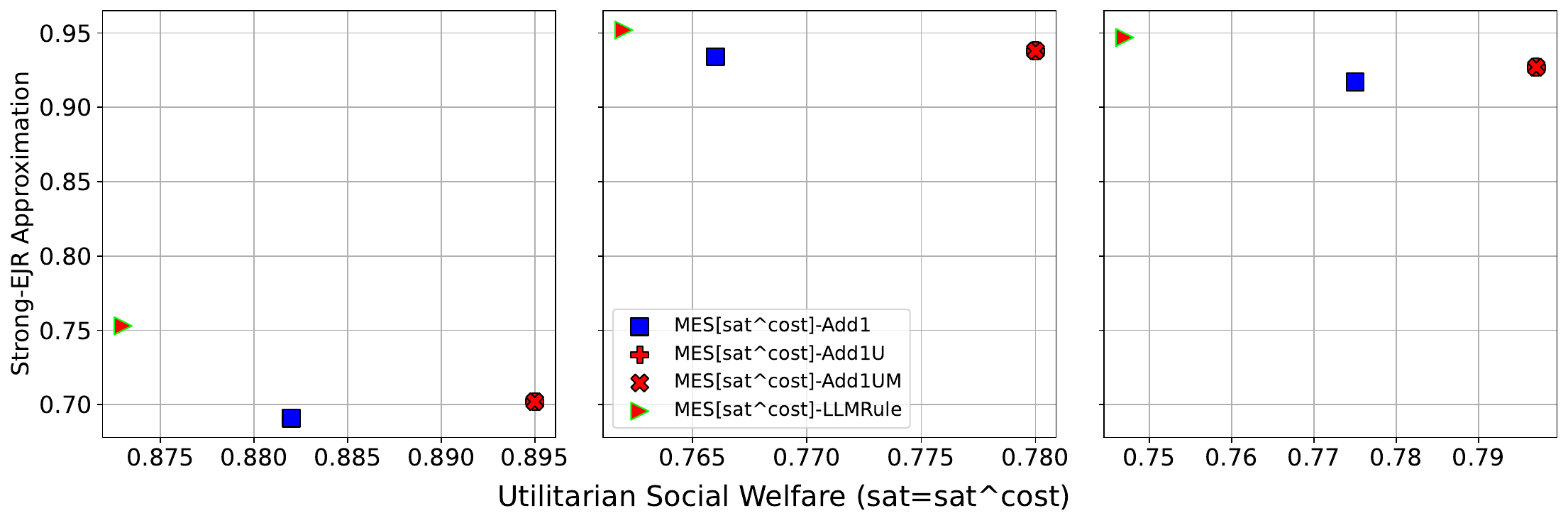}
    \end{subfigure}\\
    \begin{subfigure}[h]{0.98\textwidth}
        \centering
        \footnotesize
        \includegraphics[width=0.99\textwidth]{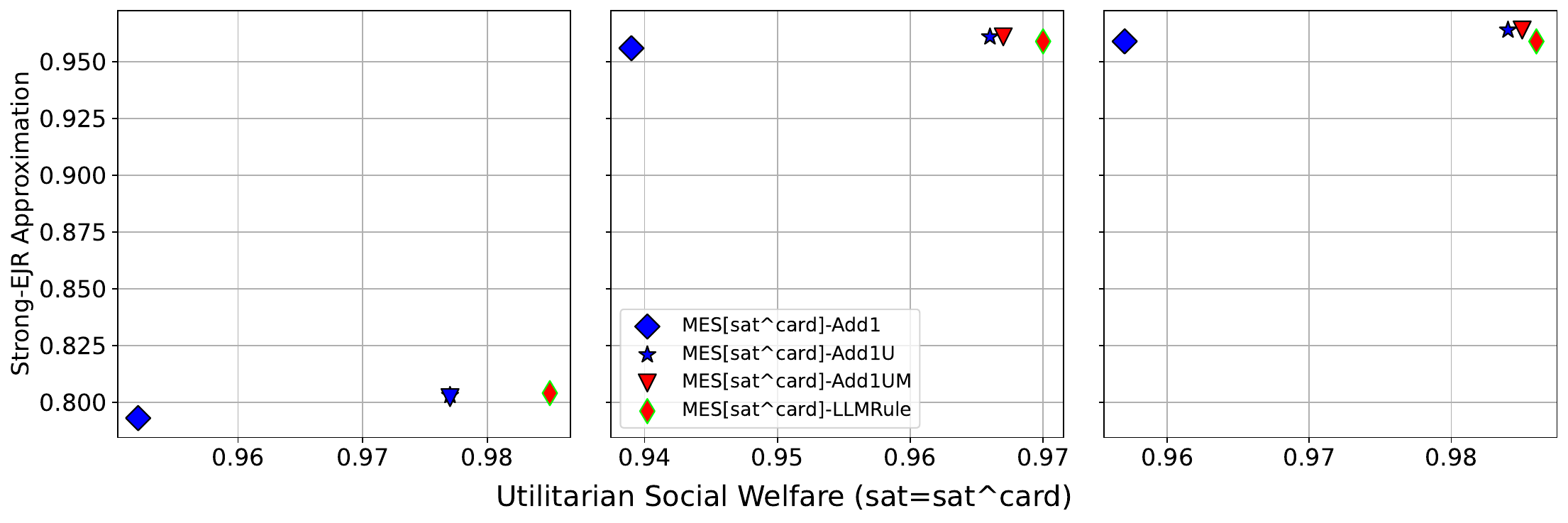}
    \end{subfigure}
    \caption{Utilitarian social welfare vs. Strong-EJR approximation of various completion methods for approval ballots on (left) synthetic, (center) Pabulib's small and (right) Pabulib's large test sets. The Pareto front is highlighted red.}
    \label{fig:tradeoff-app-completion-syn}
\end{figure}

\clearpage
\section{Limitations and Future Work}\label{app:limit}

We consider PB instances from Pabulib$^{\ref{fn:pabulib}}$ with at most 25 projects in the experiments, which comprise roughly 67\% and 92\% of all currently available instances with approval and cardinal ballots, respectively.
Despite the significant reduction in complexity for computing Strong-EJR approximation, our techniques described in Section \ref{subsec:fairness-verify} may not apply for very large $m$ due to the use of frequent itemset mining algorithms for identifying all maximal cohesive groups.
Given that we employ the classical Apriori algorithm \cite{huang2000fast}, which runs in $O(2^m)$, future works could consider more advanced algorithms to enhance efficiency and in turn enable LLMRule to work on PB instances with larger $m$.
Additionally, leveraging LLMRule to generate PB rules with provable fairness guarantees could be a promising extension.


\end{document}